\documentclass{article}

\PassOptionsToPackage{numbers, compress}{natbib}
\usepackage[final]{neurips_2021}

\usepackage[utf8]{inputenc} %
\usepackage[T1]{fontenc}    %
\usepackage{url}            %
\usepackage{booktabs}       %
\usepackage{amsfonts}       %
\usepackage{nicefrac}       %
\usepackage{microtype}      %
\usepackage{xcolor}         %
\usepackage{enumitem}
\usepackage{stackengine}
\usepackage{wrapfig}
\usepackage[colorlinks=true,linkcolor=blue,citecolor=blue]{hyperref}

\usepackage{amsmath,amssymb,amsthm}
\usepackage{algorithm,algorithmic}
\usepackage{mathtools}

\newtheorem{assumption}{Assumption}
\newtheorem{theorem}{Theorem}
\newtheorem{lemma}{Lemma}[section]
\newtheorem{corollary}{Corollary}[section]
\newtheorem{remark}[theorem]{Remark}
\newtheorem{proposition}{Proposition}[section]

\newcommand{\neurips}[1] {{\color{black} #1}}
\newcommand{\iclr}[1] {{\color{black} #1}}

\newcommand{\data}{\mathcal{D}}

\newcommand{\policy}{\pi}
\newcommand{\mdp}{\mathcal{M}}
\newcommand{\states}{\mathcal{S}}
\newcommand{\actions}{\mathcal{A}}

\newcommand{\transitions}{T}

\newcommand{\behavior}{{\pi_\beta}}
\newcommand{\bellman}{\mathcal{B}}
\newcommand{\bellmanhat}{\widehat{\mathcal{B}}}

\newcommand{\mhat}{\hat{\mathcal{M}}}
\newcommand{\mdphat}{\widehat{\mathcal{M}}}
\newcommand{\mdpbar}{\overline{\mathcal{M}}}

\newcommand{\bo}{\mathbf{o}}
\newcommand{\bs}{\mathbf{s}}
\newcommand{\ba}{\mathbf{a}}

\usepackage{amsmath,amsfonts,bm}

\def\eqref#1{equation~\ref{#1}}

\def\1{\bm{1}}

\DeclareMathAlphabet{\mathsfit}{\encodingdefault}{\sfdefault}{m}{sl}
\SetMathAlphabet{\mathsfit}{bold}{\encodingdefault}{\sfdefault}{bx}{n}

\newcommand{\E}{\mathbb{E}}

\usepackage{titlesec}
\titlespacing\section{0pt}{0pt plus 2pt minus 2pt}{0pt plus 2pt minus 2pt}
\titlespacing\subsection{0pt}{3pt plus 4pt minus 2pt}{0pt plus 2pt minus 2pt}
\titlespacing\subsubsection{0pt}{3pt plus 4pt minus 2pt}{0pt plus 2pt minus 2pt}

\title{COMBO: Conservative Offline Model-Based\\Policy Optimization}

\author{Tianhe Yu$^{*, 1}$, Aviral Kumar$^{*, 2}$, Rafael Rafailov$^{1}$, Aravind Rajeswaran$^{3}$, \vspace{0.05cm}\\ \textbf{Sergey Levine$^{2}$, Chelsea Finn$^{1}$} \vspace{0.1cm}\\
$^1$Stanford University, $^2$UC Berkeley, $^3$Facebook AI Research~~~~~~~~ ($^*$Equal Contribution) \vspace{0.1cm}\\ 
\texttt{tianheyu@cs.stanford.edu, aviralk@berkeley.edu}
}

\begin{document}

\maketitle

\begin{abstract}
  \neurips{Model-based reinforcement learning (RL) algorithms, which learn a dynamics model from logged experience and perform conservative planning under the learned model, have emerged as a promising paradigm for offline reinforcement learning (offline RL). However, practical variants of such model-based algorithms rely on explicit uncertainty quantification for incorporating conservatism. Uncertainty estimation with complex models, such as deep neural networks, can be difficult and unreliable. We empirically find that uncertainty estimation is not accurate and leads to poor performance in certain scenarios in offline model-based RL. We overcome this limitation by developing a new model-based offline RL algorithm, COMBO, that trains a value function using both the offline dataset and data generated using rollouts under the model while also additionally regularizing the value function on out-of-support state-action tuples generated via model rollouts. This results in a conservative estimate of the value function for out-of-support state-action tuples, without requiring explicit uncertainty estimation.
Theoretically, we show that COMBO satisfies a policy improvement guarantee in the offline setting.
Through extensive experiments, we find that COMBO attains greater performance compared to prior offline RL on problems that demand generalization to related but previously unseen tasks, and also consistently matches or outperforms prior offline RL methods on widely studied offline RL benchmarks, including image-based tasks.}
\end{abstract}

\section{Introduction}

Offline reinforcement learning (offline RL)~\citep{LangeGR12, levine2020offline} refers to the setting where policies are trained using static, previously collected datasets. This presents an attractive paradigm for data reuse and safe policy learning in many applications, such as healthcare~\citep{Wang2018SupervisedRL}, autonomous driving~\citep{Yu2020BDD100KAD}, robotics~\citep{kalashnikov2018scalable, Rafailov2020LOMPO}, and personalized recommendation systems~\citep{SwaminathanJ15}. Recent studies have observed that RL algorithms originally developed for the online or interactive paradigm perform poorly in the offline case~\citep{fujimoto2018off, kumar2019stabilizing, kidambi2020morel}. This is primarily attributed to the distribution shift that arises over the course of learning between the offline dataset and the learned policy. Thus, development of algorithms specialized for offline RL is of paramount importance to benefit from the offline data available in aformentioned applications. In this work, we develop a principled model-based offline RL algorithm that matches or exceeds the performance of prior offline RL algorithms in benchmark tasks.

A major paradigm for algorithm design in offline RL is to incorporate conservatism or regularization into online RL algorithms.
Model-free offline RL algorithms~\citep{fujimoto2018addressing, kumar2019stabilizing, wu2019behavior,jaques2019way,kumar2020conservative,kostrikov2021offline}
directly incorporate conservatism into the policy or value function training
and do not require learning a dynamics model. 
However, model-free algorithms learn only on the states in the offline dataset, which can lead to overly conservative algorithms.
In contrast, model-based algorithms~\citep{kidambi2020morel, yu2020mopo} learn a pessimistic dynamics model, which in turn induces a conservative estimate of the value function. By generating and training on additional synthetic data, model-based algorithms have the potential for broader generalization and solving new tasks using the offline dataset~\citep{yu2020mopo}.
\neurips{However, these methods rely on some sort of strong assumption about uncertainty estimation, typically assuming access to a \emph{model error oracle} that can estimate upper bounds on model error for any state-action tuple. In practice, such methods use more heuristic uncertainty estimation methods, which can be difficult or unreliable for complex datasets or deep network models. It then remains an open question as to whether we can formulate principled model-based offline RL algorithms with concrete theoretical guarantees on performance \emph{without} assuming access to an uncertainty or model error oracle. In this work, we propose precisely such a method, by eschewing direct uncertainty estimation, which we argue is not necessary for offline RL.}

\begin{wrapfigure}{r}{0.5\textwidth}
    \centering
    \vspace{-0.2cm}
    \includegraphics[width=0.45\textwidth]{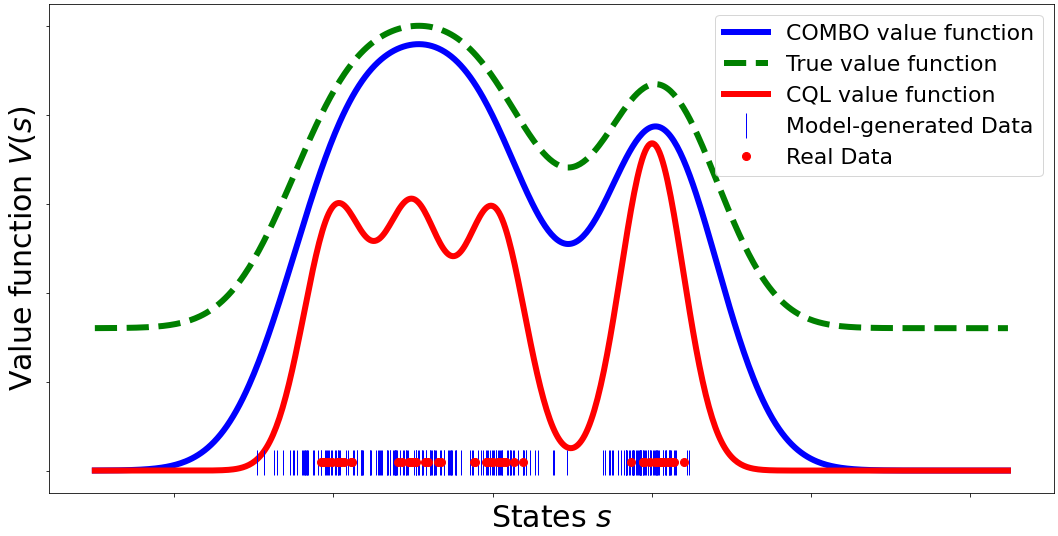}
    \vspace*{-0.4cm}
    \caption{\small COMBO learns a conservative value function by utilizing both the offline dataset as well as simulated data from the model. \neurips{Crucially, COMBO does not require uncertainty quantification, and the value function learned by COMBO is less conservative on the transitions seen in the dataset than CQL.} This enables COMBO to steer the agent towards higher value states compared to CQL, which may steer towards more optimal states,
    as illustrated in the figure.}
    \vspace*{-0.4cm}
    \label{fig:teaser}
\end{wrapfigure}

Our main contribution is the development of conservative offline model-based policy optimization (COMBO), a new model-based algorithm for offline RL. COMBO
learns a dynamics model using the offline dataset. Subsequently, it employs an actor-critic
method where the value function is learned using both the offline dataset as well as synthetically generated data from the model, similar to Dyna~\citep{sutton1991dyna} and a number of recent methods~\citep{janner2019trust,yu2020mopo,clavera2020model,Rafailov2020LOMPO}.
However, in contrast to Dyna, COMBO learns a conservative critic function
by penalizing the value function in state-action tuples that are not in the support of the offline dataset, obtained by simulating the learned model. We theoretically show that for any policy, the Q-function learned by COMBO is a lower-bound on the true Q-function.
While the approach of optimizing a performance lower-bound is similar in spirit to prior model-based algorithms~\citep{kidambi2020morel, yu2020mopo},
\neurips{COMBO crucially does not assume access to a model error or uncertainty oracle.}
\neurips{In addition, we show theoretically that the Q-function learned by COMBO is less conservative than model-free counterparts such as CQL~\citep{kumar2020conservative}, and quantify conditions under which the this lower bound is tighter than the one derived in CQL. This is illustrated through an example in Figure~\ref{fig:teaser}.  Following prior works~\citep{laroche2019safe}, we show that COMBO enjoys a safe policy improvement guarantee. By interpolating model-free and model-based components, this guarantee can utilize the best of both guarantees in certain cases.}
Finally, in our experiments, \neurips{we find that COMBO achieves the best performance on tasks that require out-of-distribution generalization and outperforms previous latent-space offline model-based RL methods on image-based robotic manipulation benchmarks. We also test COMBO on commonly studied benchmarks for offline RL and find that COMBO generally performs well on the benchmarks, achieving the highest score in $9$ out of $12$ MuJoCo domains from the D4RL~\citep{fu2020d4rl} benchmark suite.}%

\section{Preliminaries}
\label{sec:prelim}

\textbf{Markov Decision Processes and Offline RL.} We study RL in the framework of Markov decision processes (MDPs) specified by the tuple $\mathcal{M} = (\mathcal{S}, \mathcal{A}, \transitions, r, \mu_0, \gamma)$.
$\mathcal{S}, \mathcal{A}$ denote the state and action spaces. 
$\transitions(\bs' | \bs, \ba)$ and $r(\bs,\ba) \in [-R_{\max}, R_{\max}]$ represent the dynamics and reward function respectively. 
$\mu_0(s)$ denotes the initial state distribution, and $\gamma \in (0,1)$ denotes the discount factor. 
We denote the discounted state visitation distribution of a policy $\pi$ using \hbox{$d_\mathcal{M}^\pi(\bs) := (1 - \gamma)\sum_{t=0}^\infty\gamma^t\mathcal{P}(s_t = \bs | \pi)$}, where $\mathcal{P}(s_t = \bs | \pi)$ is the probability of reaching state $\bs$ at time $t$ by rolling out $\pi$ in $\mathcal{M}$. Similarly, we denote the state-action visitation distribution with $d_\mathcal{M}^\pi(\bs, \ba) \coloneqq d_\mathcal{M}^\pi(\bs)\pi(\ba|\bs)$. The goal of RL is to learn a policy that maximizes the return, or long term cumulative rewards:
$
    \max_\policy J(\mathcal{M}, \policy) := \frac{1}{1-\gamma}\E_{(\bs, \ba) \sim d_\mathcal{M}^\pi(\bs, \ba)}[r(\bs, \ba)].
$

Offline RL is the setting where we have access only to a fixed dataset $\mathcal{D} = \{(\bs, \ba, r, \bs')\}$, which consists of transition tuples from trajectories collected using a behavior policy $\behavior$. In other words, the dataset $\mathcal{D}$ is sampled from $d^\behavior(\bs, \ba) \coloneqq d^\behavior(\bs) \behavior(\ba|\bs)$. We define $\mdpbar$ as the empirical MDP induced by the dataset $\mathcal{D}$ and $d(\bs,\ba)$ as sampled-based version of $d^\behavior(\bs, \ba)$. In the offline setting, 
\iclr{the goal is to find the best possible policy using the fixed offline dataset.}

\textbf{Model-Free Offline RL Algorithms.} One class of approaches for solving MDPs involves the use of dynamic programming and actor-critic schemes~\citep{SuttonBook, BertsekasBook}, which do not explicitly require the learning of a dynamics model. To capture the long term behavior of a policy without a model, we define the action value function as
$
    Q^\pi(\bs, \ba) \coloneqq \mathbb{E} \left[ \sum_{t=0}^\infty \gamma^t \ r(\bs_t, \ba_t) \mid \bs_0 = \bs, \ba_0 = \ba \right], %
$
where future actions are sampled from $\policy(\cdot|\bs)$ and state transitions happen according to the MDP dynamics. Consider the following Bellman operator:
$
\mathcal{B}^\policy Q (\bs, \ba) := r(\bs, \ba) + \gamma \mathbb{E}_{\ \bs' \sim \transitions(\cdot | \bs, \ba), \ba' \sim \policy(\cdot | \bs')} \left[ Q(\bs', \ba') \right]$,
and its sample based counterpart:
$
\widehat{\mathcal{B}}^\policy Q (\bs, \ba) := r(\bs, \ba) + \gamma Q(\bs', \ba'),
$
associated with a single transition $(\bs, \ba, \bs')$ and $\ba' \sim \policy (\cdot | \bs')$. The action-value function satisfies the Bellman consistency criterion given by $\mathcal{B}^\policy Q^\policy(\bs, \ba) = Q^\policy (\bs, \ba) \ \forall (\bs, \ba)$. When given an offline dataset $\mathcal{D}$, standard approximate dynamic programming~(ADP) and actor-critic methods use this criterion to alternate between policy evaluation~\citep{Munos2008FiniteTimeBF} and policy improvement. A number of prior works have observed that such a direct extension of ADP and actor-critic schemes to offline RL leads to poor results due to distribution shift over the course of learning and over-estimation bias in the $Q$ function~\citep{fujimoto2018off, kumar2019stabilizing, wu2019behavior}. To address these drawbacks, prior works have proposed a number of modifications aimed towards regularizing the policy or value function (see Section~\ref{sec:related}). In this work, we primarily focus on CQL~\citep{kumar2020conservative}, which alternates between:

{\bf Policy Evaluation:} The $Q$ function associated with the current policy $\policy$ is approximated conservatively by repeating the following optimization:

\vspace*{-15pt}
\begin{small}
\begin{align}
    \!\!\!\!{Q}^{k+1}\!\!\leftarrow&\! \arg\min_{Q} \beta\!\left(\E_{\bs \sim \mathcal{D}, \ba \sim \mu(\cdot|\bs)}\!\left[Q(\bs,\ba)\right]\!-\!\E_{\bs, \ba \sim \data}\!\left[Q(\bs,\ba)\right]\right)\!+\!\frac{1}{2}\E_{\bs, \ba, \bs' \sim \mathcal{D}}\!\!\left[\!\left(Q(\bs, \ba)\! -\! \widehat{\bellman}^\policy{Q}^k(\bs, \ba)\right)\!^2\! \right]\!,
    \label{eq:cql_loss}
\end{align}
\end{small}
\vspace*{-15pt}

where $\mu(\cdot|s)$ is a wide sampling distribution such as the uniform distribution over action bounds. CQL effectively penalizes the $Q$ function at states in the dataset for actions not observed in the dataset. This enables a conservative estimation of the value function for any policy~\citep{kumar2020conservative}, mitigating the challenges of over-estimation bias and distribution shift.

{\bf Policy Improvement:} After approximating the Q function as $\hat{Q}^\policy$, the policy is improved as
$
\policy \leftarrow \arg \max_{\policy'}\!\! \ \mathbb{E}_{\bs \sim \mathcal{D}, \ba \sim \policy'(\cdot | \bs)}\!\left[ \hat{Q}^\policy (\bs, \ba) \right].
$
Actor-critic methods with parameterized policies and $Q$ functions approximate $\arg \max$ and $\arg \min$ in above equations with a few gradient descent steps.

\textbf{Model-Based Offline RL Algorithms.} A second class of algorithms for solving MDPs involve the learning of the dynamics function, and using the learned model to aid policy search. Using the given dataset $\mathcal{D}$, a dynamics model $\widehat{T}$ is typically trained using maximum likelihood estimation as:
$\min_{\widehat{T}} \ \mathbb{E}_{(\bs, \ba, \bs') \sim \mathcal{D}} \left[ \log \widehat{T}(\bs' | \bs, \ba) \right]$.
A reward model $\hat{r}(\bs, \ba)$ can also be learned similarly if it is unknown. Once a model has been learned, we can construct the learned MDP $\widehat{\mathcal{M}} = (\mathcal{S}, \mathcal{A}, \widehat{T}, \hat{r}, \mu_0, \gamma)$, which has the same state and action spaces, but uses the learned dynamics and reward function. 
Subsequently, any policy learning or planning algorithm can be used to recover the optimal policy in the model as $\hat{\policy} = \arg \max_{\policy} J(\widehat{\mathcal{M}}, \policy).$

This straightforward approach is known to fail in the offline RL setting, both in theory and practice, due to distribution shift and model-bias~\citep{RossB12, kidambi2020morel}. In order to overcome these challenges, offline model-based algorithms like MOReL~\citep{kidambi2020morel} and MOPO~\citep{yu2020mopo} use uncertainty quantification to construct a lower bound for policy performance and optimize this lower bound by assuming a model error oracle $u(\bs,\ba)$. By using an uncertainty estimation algorithm like bootstrap ensembles~\citep{OsbandAC18, Azizzadenesheli18, POLO}, we can estimate $u(\bs, \ba)$.
By constructing and optimizing such a lower bound, offline model-based RL algorithms avoid the aforementioned pitfalls like model-bias and distribution shift.
While any RL or planning algorithm can be used to learn the optimal policy for $\widehat{\mathcal{M}}$, we focus specifically on MBPO~\citep{janner2019trust, sutton1991dyna} which was used in MOPO. MBPO follows the standard structure of actor-critic algorithms, but in each iteration uses an augmented dataset $\data \cup \data_{\text{model}}$ for policy evaluation. Here, $\data$ is the offline dataset and $\data_{\text{model}}$ is a dataset obtained by simulating the current policy using the learned dynamics model. 
Specifically, at each iteration, MBPO performs $k$-step rollouts using $\widehat{T}$ starting from state $\bs \in \mathcal{D}$ with a particular rollout policy $\mu(\ba|\bs)$, adds the model-generated data to $\data_{\text{model}}$, and optimizes the policy with a batch of data sampled from $\mathcal{D} \cup \data_{\text{model}}$ where each datapoint in the batch is drawn from $\mathcal{D}$ with probability $f \in [0, 1]$ and $\data_{\text{model}}$ with probability $1 - f$.

\section{Conservative Offline Model-Based Policy Optimization}
\label{sec:combo}

\neurips{The principal limitation of prior offline model-based algorithms (discussed in Section~\ref{sec:prelim}) is the assumption of having access to a model error oracle for uncertainty estimation and strong reliance on heuristics of quantifying the uncertainty. In practice, such heuristics could be challenging for complex datasets or deep neural network models~\citep{ovadia2019can}. We argue that uncertainty estimation is not imperative for offline model-based RL and empirically show that uncertainty estimation could be inaccurate in offline RL problems especially when generalization to unknown behaviors is required in Section~\ref{sec:uq}. Our goal is to develop a model-based offline RL algorithm that enables optimizing a lower bound on the policy performance, but without requiring uncertainty quantification. We achieve this by extending conservative Q-learning~\citep{kumar2020conservative}, which does not require explicit uncertainty quantification, into the model-based setting. 
Our algorithm COMBO, summarized in Algorithm~\ref{alg:combo}, alternates between a conservative policy evaluation step and a policy improvement step, which we outline below.}

{\bf Conservative Policy Evaluation:} Given a policy $\policy$, an offline dataset $\data$, and a learned model of the MDP $\mhat$, the goal in this step is to obtain a conservative estimate of $Q^\policy$. To achieve this, we penalize the Q-values evaluated on data drawn from a particular state-action distribution that is more likely to be out-of-support while pushing up the Q-values on state-action pairs that are trustworthy, which is implemented by repeating the following recursion:

\vspace*{-10pt}
\begin{small}
\begin{align}
    \!\!\hat{Q}^{k+1}\! \leftarrow&\! \arg\min_{Q} \beta\left(\E_{\bs, \ba \sim \rho(\bs,\ba)}\!\left[Q(\bs,\ba)\right]\!-\!\E_{\bs, \ba \sim \data}\!\left[Q(\bs,\ba)\right]\right)\!+\!\frac{1}{2}\E_{\bs, \ba, \bs' \sim d_f}\!\!\left[\!\left(Q(\bs, \ba)\! -\! \widehat{\bellman}^\policy\hat{Q}^k(\bs, \ba)\right)^2\! \right]\!.
    \label{eq:implicit_update}
\end{align}
\end{small}
\vspace*{-10pt}

Here, $\rho(\bs, \ba)$ and $d_f$ are sampling distributions that we can choose. Model-based algorithms allow ample flexibility for these choices while providing the ability to control the bias introduced by these choices. For $\rho(\bs, \ba)$, we make the following choice:
$
\rho(\bs, \ba) =  d^\policy_{\mdphat} (\bs) \pi(\ba | \bs),
$
where $d^\policy_{\mdphat} (\bs)$ is the discounted marginal state distribution when executing $\policy$ in the learned model $\mdphat$. Samples from $d^\policy_{\mdphat} (\bs)$ can be obtained by rolling out $\policy$ in $\mdphat$.
Similarly, $d_f$ is an $f-$interpolation between the offline dataset and synthetic rollouts from the model:
$
d_f^\mu (\bs, \ba) := f \ d(\bs, \ba) + (1-f) \ d^\mu_{\mdphat} (\bs, \ba),
$
where $f \in [0,1]$ is the ratio of the datapoints drawn from the offline dataset as defined in Section~\ref{sec:prelim} and $\mu(\cdot | \bs)$ is the rollout distribution used with the model, which can be modeled as $\policy$ or a uniform distribution. To avoid notation clutter, we also denote $d_f := d_f^\mu$.

Under such choices of $\rho$ and $d_f$, we push down (or conservatively estimate) Q-values on state-action tuples from model rollouts and push up Q-values on the real state-action pairs from the offline dataset. When updating Q-values with the Bellman backup, we use a mixture of both the model-generated data and the real data, similar to Dyna~\citep{sutton1991dyna}. 
Note that in comparison to CQL and other model-free algorithms, COMBO learns the Q-function over a richer set of states beyond the states in the offline dataset. 
This is made possible by performing rollouts under the learned dynamics model, denoted by $d^\mu_{\mdphat} (\bs, \ba)$.
We will show in Section~\ref{sec:theory} that the Q function learned by repeating the recursion in Eq.~\ref{eq:implicit_update} provides a lower bound on the true Q function, without the need for explicit uncertainty estimation. Furthermore, we will theoretically study the advantages of using synthetic data from the learned model, and characterize the impacts of model bias.

\begin{algorithm}[t!]
\begin{small}
  \caption{COMBO: Conservative Model Based Offline Policy Optimization}\label{alg:combo}
  \begin{algorithmic}[1]
    \REQUIRE Offline dataset $\data$, rollout distribution $\mu(\cdot|\bs)$, learned dynamics model $\widehat{T}_\theta$, initialized policy and critic  $\policy_\phi$ and $Q_\psi$.
    \STATE Train the probabilistic dynamics model $\widehat{T}_\theta(\bs', r|\bs,\ba) = \mathcal{N}(\mu_\theta(\bs, \ba), \Sigma_\theta(\bs, \ba))$ on $\data$.
    \STATE Initialize the replay buffer $\data_{\text{model}} \leftarrow \varnothing$.
    \FOR{$i=1, 2, 3, \cdots,$}
    \STATE Collect model rollouts by sampling from $\mu$ and $\widehat{T}_\theta$ starting from states in $\data$. Add model rollouts to $\data_\text{model}$.
    \STATE Conservatively evaluate $\policy_\phi^{i}$ by repeatedly solving eq.~\ref{eq:implicit_update} to obtain $\hat{Q}^{\policy_\phi^i}_\psi$ using samples from $\data \cup \data_\text{model}$.
    \STATE Improve policy under state marginal of $d_f$ by solving eq.~\ref{eq:combo_policy_improvement} to obtain $\policy_\phi^{i+1}$.
    \ENDFOR
  \end{algorithmic}
\end{small}
\end{algorithm}

{\bf Policy Improvement Using a Conservative Critic:} After learning a conservative critic $\hat{Q}^\policy$, we improve the policy as:
\begin{equation}
\label{eq:combo_policy_improvement}
\policy' \leftarrow \arg \max_{\policy} \ \mathbb{E}_{\bs \sim \rho, \ba \sim \policy(\cdot|\bs)} \left[ \hat{Q}^{\policy} (\bs, \ba) \right]
\end{equation}
where $\rho(\bs)$ is the state marginal of $\rho(\bs,\ba)$. When policies are parameterized with neural networks, we approximate the $\arg \max$ with a few steps of gradient descent. In addition, entropy regularization can also be used to prevent the policy from becoming degenerate if required~\citep{haarnoja2018soft}. In Section~\ref{sec:policy_improvement_theory}, we show that the resulting policy is guaranteed to improve over the behavior policy.

\textbf{Practical Implementation Details.} Our practical implementation largely follows MOPO, with the key exception that we perform conservative policy evaluation as outlined in this section, rather than using uncertainty-based reward penalties. Following MOPO, we represent the probabilistic dynamics model using a neural network, with parameters $\theta$, that produces a Gaussian distribution over the next state and reward: $\widehat{T}_\theta(\bs_{t+1}, r| \bs, \ba) = \mathcal{N}(\mu_\theta(\bs_t, \ba_t), \Sigma_\theta(\bs_t, \ba_t))$. The model is trained via maximum likelihood.
For conservative policy evaluation (eq.~\ref{eq:implicit_update}) and policy improvement (eq.~\ref{eq:combo_policy_improvement}), we augment $\rho$ with states sampled from the offline dataset, which shows more stable improvement in practice.
\iclr{It is relatively common in prior work on model-based offline RL to select various hyperparameters using online policy rollouts~\citep{yu2020mopo,kidambi2020morel,argenson2020model,lee2021representation}. However, we would like to avoid this with our method, since requiring online rollouts to tune hyperparameters contradicts the main aim of offline RL, which is to learn entirely from offline data. Therefore, \emph{we do not use online rollouts for tuning COMBO}, and instead devise an automated rule for tuning important hyperparameters such as $\beta$ and $f$ in a fully offline manner. We search over a small discrete set of hyperparameters for each task, and use the value of the regularization term $\mathbb{E}_{\mathbf{s}, \mathbf{a} \sim \rho(\mathbf{s},\mathbf{a})}\!\left[Q(\mathbf{s},\mathbf{a})\right]\!-\!\mathbb{E}_{\mathbf{s}, \mathbf{a} \sim \data}\!\left[Q(\mathbf{s},\mathbf{s})\right]$ (shown in Eq.~\ref{eq:implicit_update}) to pick hyperparameters in an entirely offline fashion. We select the hyperparameter setting that achieves the lowest regularization objective, which indicates that the Q-values on unseen model-predicted state-action tuples are not overestimated.}
Additional details about the practical implementation and the hyperparameter selection rule are provided in Appendix~\ref{app:combo_details} and Appendix~\ref{app:hyperparameter} respectively.

\section{Theoretical Analysis of COMBO}
\label{sec:theory}
In this section, we theoretically analyze our method and show that it optimizes a lower-bound on the expected return of the learned policy. This lower bound is close to the actual policy performance (modulo sampling error) when the policy's state-action marginal distribution is in support of the state-action marginal of the behavior policy and conservatively estimates the performance of a policy otherwise. By optimizing the policy against this lower bound, COMBO guarantees policy improvement beyond the behavior policy. \neurips{Furthermore, we use these insights to discuss cases when COMBO is less conservative compared to model-free counterparts}.

\begin{subsection}{COMBO Optimizes a Lower Bound}
\label{sec:combo_lower_bound}
We first show that training the Q-function using Eq.~\ref{eq:implicit_update} produces a Q-function such that the expected off-policy policy improvement objective~\citep{degris2012off} 
computed using this learned Q-function lower-bounds its actual value. We will reuse notation for $d_f$ and $d$ from Sections~\ref{sec:prelim} and~\ref{sec:combo}. 
Assuming that the Q-function is tabular, the Q-function found by approximate dynamic programming in iteration $k$, can be obtained by differentiating Eq.~\ref{eq:implicit_update} with respect to $Q^k$ (see App.~\ref{app:proofs} for details):
\begin{equation}
    \hat{Q}^{k+1}(\bs, \ba) = (\bellmanhat^\pi Q^k)(\bs, \ba) - \beta \frac{\rho(\bs, \ba) - d(\bs, \ba)}{d_f(\bs, \ba)}.
\label{eqn:combo_iterate}
\end{equation}
Eq.~\ref{eqn:combo_iterate} effectively applies a penalty that depends on the three distributions appearing in the COMBO critic training objective (Eq.~\ref{eq:implicit_update}), of which $\rho$ and $d_f$ are free variables that we choose in practice as discussed in Section~\ref{sec:combo}. For a given iteration $k$ of Eq.~\ref{eqn:combo_iterate}, we further define the expected penalty under $\rho(\bs, \ba)$ as: 
\begin{equation}
 \nu(\rho, f) := \E_{\bs, \ba \sim \rho(\bs, \ba)}\left[\frac{\rho(\bs, \ba) - d(\bs, \ba)}{d_f(\bs, \ba)} \right]\label{eqn:expected_penalty}.
\end{equation}

Next, we will show that the Q-function learned by COMBO lower-bounds the actual Q-function under the initial state distribution $\mu_0$ and any policy $\pi$. We also show that the asymptotic Q-function learned by COMBO lower-bounds the actual Q-function of any policy $\pi$ with high probability for a large enough $\beta \geq 0$, which we include in Appendix~\ref{app:proof_lower_bound}. Let $\mdpbar$ represent the empirical MDP which uses the empirical transition model based on raw data counts. The Bellman backups over the dataset distribution $d_f$ in Eq.~\ref{eq:implicit_update} that we analyze is an $f-$interpolation 
of the backup operator in the empirical MDP (denoted by $\bellman_{\mdpbar}^\pi$) and the backup operator under the learned model $\mdphat$ (denoted by $\bellman_{\mdphat}^\pi$).
The empirical backup operator suffers from sampling error, but is unbiased in expectation, whereas the model backup operator induces bias but no sampling error.
We assume that all of these backups enjoy concentration properties with concentration coefficient $C_{r, T, \delta}$, dependent on the desired confidence value $\delta$ (details in Appendix~\ref{app:proof_lower_bound}). This is a standard assumption in literature~\citep{laroche2019safe}.
Now, we state our main results below.
\begin{proposition}
\label{thm:lower_bound}
For large enough $\beta$, we have
$\E_{\bs \sim \mu_0, \ba \sim \policy(\cdot|\bs)}[\hat{Q}^\pi(\bs, \ba)] \leq \E_{\bs \sim \mu_0, \ba \sim \policy(\cdot|\bs)}[Q^\pi(\bs, \ba)]$, 
where $\mu_0(\bs)$ is the initial state distribution. 
Furthermore, when $\epsilon_{\text{s}}$ is small, such as in the large sample regime, or when the model bias $\epsilon_{\text{m}}$ is small, a small $\beta$ is sufficient to guarantee this condition along with an appropriate choice of $f$.
\end{proposition}

The proof for Proposition~\ref{thm:lower_bound} can be found in Appendix~\ref{app:proof_lower_bound}.
Finally, while \citet{kumar2020conservative} also analyze how regularized value function training can provide lower bounds on the value function at each state in the dataset~\citep{kumar2020conservative} (Proposition 3.1-3.2), \neurips{our result shows that COMBO is less conservative in that it does not underestimate the value function at every state in the dataset like CQL (Remark~\ref{remak:tighter_lower_bound}) and might even overestimate these values. Instead COMBO penalizes Q-values at states generated via model rollouts from $\rho(\bs, \ba)$. 
\iclr{Note that in general, the required value of $\beta$ may be quite large similar to prior works, which typically utilize a large constant $\beta$, which may be in the form of a penalty on a regularizer~\citep{liu2020provably,kumar2020conservative} or as constants in theoretically optimal algorithms~\citep{jin2021pessimism,rashidinejad2021bridging}.}
While it is challenging to argue that that either COMBO or CQL attains the tightest possible lower-bound on return, in our final result of this section, we discuss a sufficient condition for the COMBO lower-bound to be tighter than CQL. 
\begin{proposition}
\label{prop:less_conservative}
Assuming previous notation, let $\Delta^\pi_{\text{COMBO}} := \E_{\bs, \ba \sim d_{\mdpbar}(\bs), \pi(\ba|\bs)}\left[ \hat{Q}^\pi(\bs, \ba) \right]$ and $\Delta^\pi_{\text{CQL}} := \E_{\bs, \ba \sim d_{\mdpbar}(\bs), \pi(\ba|\bs)}\left[ \hat{Q}^\pi_\text{CQL}(\bs, \ba) \right]$ denote the average values on the dataset under the Q-functions learned by COMBO and CQL respectively. 
Then, $\Delta^\pi_{\text{COMBO}} \geq \Delta^\pi_\text{CQL}$, if:
\begin{equation*}
    \E_{\bs, \ba \sim \rho(\bs,\ba)}\left[ \frac{\pi(\ba|\bs)}{\behavior(\ba|\bs)} \right] - \E_{\bs, \ba \sim d_{\mdpbar}(\bs), \pi(\ba|\bs)}\left[ \frac{\pi(\ba|\bs)}{\behavior(\ba|\bs)} \right] \leq 0.~~~~~~~~~~~(*)
\end{equation*}
\end{proposition}
Proposition~\ref{prop:less_conservative} indicates that COMBO will be less conservative than CQL when the action probabilities under learned policy $\pi(\ba|\bs)$ and the probabilities under the behavior policy $\pi_\beta(\ba|\bs)$ are closer together on state-action tuples drawn from $\rho(\bs, \ba)$ (i.e., sampled from the model using the policy $\pi(\ba|\bs)$), than they are on states from the dataset and actions from the policy, $d_{\mdpbar}(\bs)\pi(\ba|\bs)$.
COMBO's objective (Eq.~\ref{eq:implicit_update}) only penalizes Q-values under $\rho(\bs, \ba)$, which, in practice, are expected to primarily consist of out-of-distribution states generated from model rollouts, and does not penalize the Q-value at states drawn from $d_{\mdpbar}(\bs)$. As a result, the expression $(*)$ is likely to be negative, making COMBO less conservative than CQL.
} 

\subsection{Safe Policy Improvement Guarantees}
\label{sec:policy_improvement_theory}
Now that we have shown various aspects of the lower-bound on the Q-function induced by COMBO,
we provide policy improvement guarantees for the COMBO algorithm. Formally, Proposition~\ref{thm:policy_improvement} discuss safe improvement guarantees over the behavior policy. building on prior work~\citep{petrik2016safe,laroche2019safe,kumar2020conservative}. 
\begin{proposition}[$\zeta$-safe policy improvement]
\label{thm:policy_improvement}
Let $\hat{\pi}_{\text{out}}(\ba|\bs)$ be the policy obtained by COMBO.
Then, \iclr{if $\beta$ is sufficiently large and $\nu(\rho^\pi, f) - \nu(\rho^\beta, f) \geq C$ for a positive constant $C$}, the policy $\hat{\pi}_{\text{out}}(\ba|\bs)$ is a $\zeta$-safe policy improvement over ${\behavior}$ in the actual MDP $\mdp$, i.e., $J(\hat{\pi}_{\text{out}}, \mdp) \geq J({\behavior}, \mdp) - \zeta$, with probability at least $1 - \delta$, where $\zeta$ is given by,
\small{
\begin{align*}
    \mathcal{O}\left(\frac{\gamma f}{(1 - \gamma)^2}\right) \underbrace{\E_{\bs \sim d^{\hat{\pi}_{\text{out}}}_{\mdp}}\left[ \sqrt{\frac{|\actions|}{|\data(\bs)|} \mathrm{D}_{\text{CQL}}(\hat{\pi}_{\text{out}}, \behavior)} \right]}_{:=~ \text{(1)}} + \mathcal{O}\left(\frac{\gamma (1 - f)}{(1 - \gamma)^2}\right) \underbrace{ \mathrm{D_{TV}}(\mdpbar, \mdphat)}_{:=~ \text{(2)}} - \underbrace{\beta \frac{C}{(1 - \gamma)}}_{:=~ \text{(3)}}.
\end{align*}
}
\end{proposition}
The complete statement (with constants and terms that grow smaller than quadratic in the horizon) and proof for Proposition~\ref{thm:policy_improvement} is provided in Appendix~\ref{app:proof_policy_improvement}. $D_{\text{CQL}}$ denotes a notion of probabilistic distance between policies~\citep{kumar2020conservative} which we discuss further in Appendix~\ref{app:proof_policy_improvement}. The expression for $\zeta$ in Proposition~\ref{thm:policy_improvement} consists of three terms: term (1)~captures the decrease in the policy performance due to limited data, and decays as the size of $\mathcal{D}$ increases. The second term (2)~captures the suboptimality induced by the bias in the learned model. Finally, as we show in Appendix~\ref{app:proof_policy_improvement}, the third term (3) comes from $\nu(\rho^\pi, f) - \nu(\rho^\beta, f)$, which is equivalent to the improvement in policy performance as a result of running COMBO in the empirical and model MDPs. Since the learned model is trained on the dataset $\data$ with transitions generated from the behavior policy $\behavior$, the marginal distribution $\rho^\beta(\bs, \ba)$ is expected to be closer to $d(\bs, \ba)$ for $\behavior$ as compared to the counterpart for the learned policy, $\rho^\pi$. \iclr{Thus, the assumption that $\nu(\rho^\pi, f) - \nu(\rho^\beta, f)$ is positive is reasonable}, and in such cases, an appropriate (large) choice of $\beta$ will make term (3) large enough to counteract terms (1) and (2) that reduce policy performance. We discuss this elaborately in Appendix~\ref{app:proof_policy_improvement} (Remark~\ref{remark:remark1}). 

Further note that in contrast to Proposition 3.6 in \citet{kumar2020conservative}, note that our result indicates the sampling error (term (1)) is reduced (multiplied by a fraction $f$) when a near-accurate model is used to augment data for training the Q-function, \neurips{and similarity, it can avoid the bias of model-based methods by relying more on the model-free component. This allows COMBO to attain the best-of-both model-free and model-based methods, via a suitable choice of the fraction $f$.}

To summarize,
through an appropriate choice of $f$, Proposition~\ref{thm:policy_improvement} guarantees safe improvement over the behavior policy without requiring access to an oracle uncertainty estimation algorithm.
\end{subsection}

\section{Experiments}
\label{sec:exp}

In our experiments, we aim to answer the follow questions: \neurips{(1) Can COMBO generalize better than previous offline model-free and model-based approaches in a setting that requires generalization to tasks that are different from what the behavior policy solves?
(2) How does COMBO compare with prior work in tasks with high-dimensional image observations? (3) How does COMBO compare to prior offline model-free and model-based methods in standard offline RL benchmarks? }

To answer those questions, we compare COMBO to several prior methods. In the domains with compact state spaces, we compare with recent model-free algorithms like BEAR~\citep{kumar2019stabilizing}, BRAC~\citep{wu2019behavior}, and CQL~\citep{kumar2020conservative}; as well as MOPO~\citep{yu2020mopo} \neurips{and MOReL~\citep{kidambi2020morel} which are two recent model-based algorithms}. In addition, we also compare with an offline version of SAC~\citep{haarnoja2018soft} (denoted as SAC-off), and behavioral cloning (BC).
In high-dimensional image-based domains, which we use to answer question (3), we compare to LOMPO~\citep{Rafailov2020LOMPO}, which is a latent space offline model-based RL method that handles image inputs, latent space MBPO (denoted LMBPO), similar to \citet{janner2019trust} which uses the model to generate additional synthetic data, the fully offline version of SLAC \citep{lee2019SLAC} (denoted SLAC-off),
which only uses a variational model for state representation purposes, and CQL from image inputs. To our knowledge, CQL, MOPO, and LOMPO are representative of state-of-the-art model-free and model-based offline RL methods. Hence we choose them as comparisons to COMBO.
\iclr{To highlight the distinction between COMBO and a na\"{i}ve combination of CQL and MBPO, we perform such a comparison in Table~\ref{tbl:cql_mbpo} in Appendix~\ref{app:cql_mbpo}}. For more details of our experimental set-up, comparisons, and hyperparameters, see Appendix~\ref{app:details}.

\subsection{Results on tasks that require generalization}
\label{sec:generalization_exps}

\begin{wraptable}{r}{10.5cm}
\centering
\scriptsize
\begin{tabular}{l|r|r|r|r|r|r}
\toprule 
\textbf{Environment} & \stackanchor{\textbf{Batch}}{\textbf{Mean}} & \stackanchor{\textbf{Batch}}{\textbf{Max}} & \stackanchor{\textbf{COMBO}}{\textbf{(Ours)}} & \textbf{MOPO} & \neurips{\textbf{MOReL}} & \textbf{CQL}\\ \midrule
halfcheetah-jump & -1022.6 & 1808.6 & \textbf{5308.7}$\pm$575.5 & 4016.6 & \neurips{3228.7} & 741.1\\
ant-angle & 866.7 & 2311.9 & \textbf{ 2776.9}$\pm$43.6 & 2530.9 & \neurips{2660.3} & 2473.4\\
sawyer-door-close & 5\% & 100\% & \textbf{98.3}\%$\pm$3.0\% & 65.8\% & 42.9\% & 36.7\%\\
\bottomrule
\end{tabular}
\vspace{-0.2cm}
\caption{
\footnotesize Average returns of \texttt{halfcheetah-jump} and \texttt{ant-angle} and average success rate of \texttt{sawyer-door-close} that require out-of-distribution generalization. All results are averaged over 6 random seeds. We include the mean and max return / success rate of episodes in the batch data (under Batch Mean and Batch Max, respectively) for comparison. We also include the 95\%-confidence interval for COMBO.
}
\label{tbl:generalize}
\normalsize
\end{wraptable}

To answer question (1), we use two environments \texttt{halfcheetah-jump} and \texttt{ant-angle} constructed in \citet{yu2020mopo}, which requires the agent to solve a task that is different from what the behavior policy solved. In both environments, the offline dataset is collected by policies trained with original reward functions of \texttt{halfcheetah} and \texttt{ant}, which reward the robots to run as fast as possible. The behavior policies are trained with SAC with 1M steps and we take the full replay buffer as the offline dataset. Following \citet{yu2020mopo}, we relabel rewards in the offline datasets to reward the halfcheetah to jump as high as possible and the ant to run to the top corner with a 30 degree angle as fast as possible. Following the same manner, we construct a third task \texttt{sawyer-door-close} based on the environment in \citet{yu2020meta, Rafailov2020LOMPO}. In this task, we collect the offline data with SAC policies trained with a sparse reward function that only gives a reward of 1 when the door is \textit{opened} by the sawyer robot and 0 otherwise. The offline dataset is similar to the ``medium-expert`` dataset in the D4RL benchmark since we mix equal amounts of data collected by a fully-trained SAC policy and a partially-trained SAC policy. We relabel the reward such that it is 1 when the door is \textit{closed} and 0 otherwise. Therefore, in these datasets, the offline RL methods must generalize beyond behaviors in the offline data in order to learn the intended behaviors. We visualize the \texttt{sawyer-door-close} environment in the right image in Figure~\ref{fig:visual} in Appendix~\ref{app:image_details}.

We present the results on the three tasks in Table~\ref{tbl:generalize}. \neurips{COMBO significantly outperforms MOPO, MOReL and CQL, two representative model-based methods and one representative model-free methods respectively}, in the \texttt{halfcheetah-jump} and \texttt{sawyer-door-close} tasks, and \neurips{achieves an approximately 8\%, 4\% and 12\% improvement over MOPO, MOReL and CQL respectively on the \texttt{ant-angle} task}. These results validate that COMBO achieves better generalization results in practice by behaving less conservatively than prior model-free offline methods (compare to CQL, which doesn't improve much), and does so more robustly than prior model-based offline methods (compare to MOReL and MOPO).

\subsubsection{Empirical analysis on uncertainty estimation in offline model-based RL}
\label{sec:uq}

\begin{wrapfigure}{r}{0.65\textwidth}
\vspace{-0.5cm}
    \centering
    \includegraphics[width=0.33\textwidth]{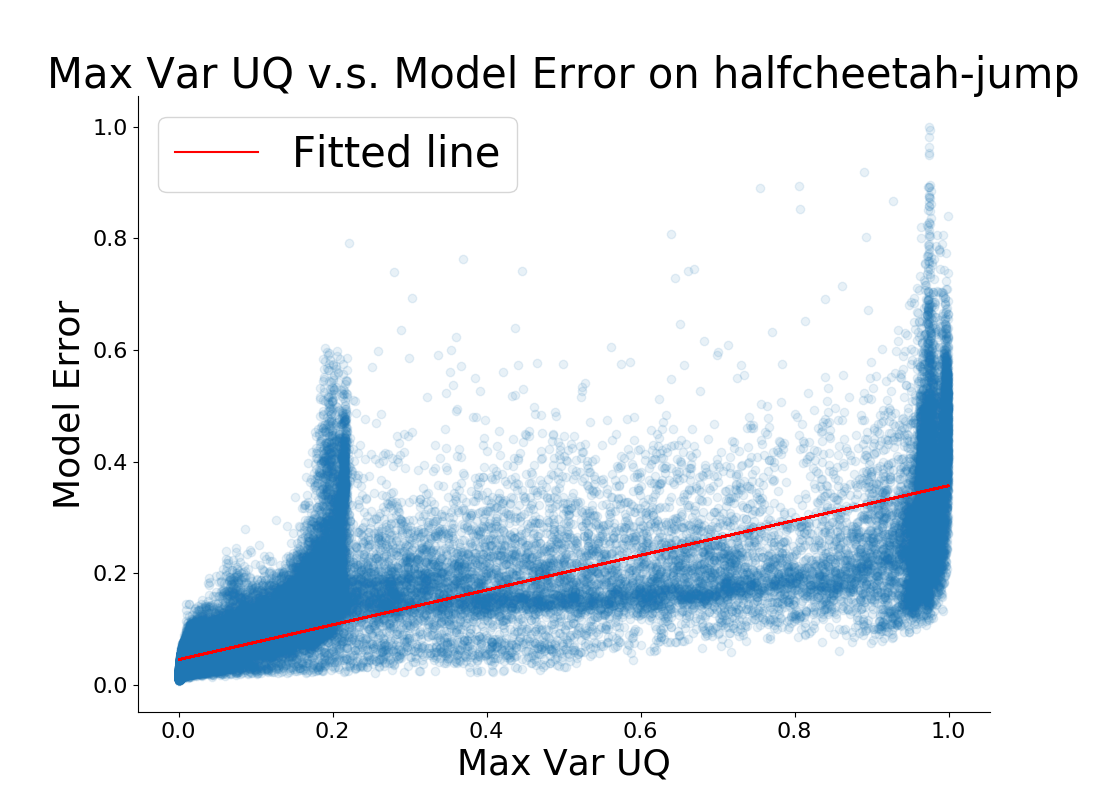}
    \includegraphics[width=0.3\textwidth]{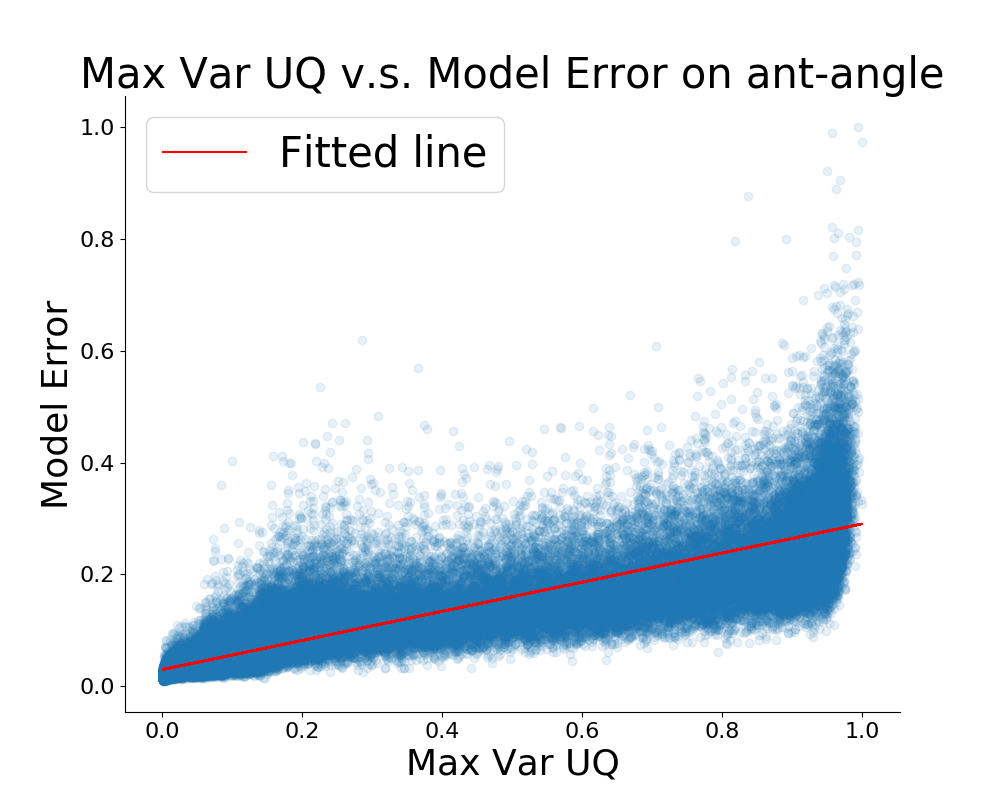}
    \vspace{-0.2cm}
    \caption{\footnotesize
    \neurips{We visualize the fitted linear regression line between the model error and two uncertainty quantification methods maximum learned variance over the ensemble (denoted as \textbf{Max Var}) on two tasks that test the generalization abilities of offline RL algorithms (\texttt{halfcheetah-jump} and \texttt{ant-angle}). We show that \textbf{Max Var} struggles to predict the true model error. Such visualizations indicates that uncertainty quantification is challenging with deep neural networks and could lead to poor performance in model-based offline RL in settings where out-of-distribution generalization is needed. In the meantime, COMBO addresses this issue by removing the burden of performing uncertainty quantification.}}
    \label{fig:uq}
    \vspace{-0.4cm}
\end{wrapfigure}

\neurips{To further understand why COMBO outperforms prior model-based methods in tasks that require generalization, we argue that one of the main reasons could be that uncertainty estimation is hard in these tasks where the agent is required to go further away from the data distribution. To test this intuition, we perform empirical evaluations to study whether uncertainty quantification with deep neural networks, especially in the setting of dynamics model learning, is challenging and could cause problems with uncertainty-based model-based offline RL methods such as MOReL~\citep{kidambi2020morel} and MOPO~\citep{yu2020mopo}. In our evaluations, we consider maximum learned variance over the ensemble (denoted as \textbf{Max Var}) $\max_{i=1,\dots,N}\|\Sigma^i_\theta(\bs,\ba)\|_\text{F}$ (used in MOPO).

We consider two tasks \texttt{halfcheetah-jump} and \texttt{ant-angle}. We normalize both the model error and the uncertainty estimates to be within scale $[0, 1]$ and performs linear regression that learns the mapping between the uncertainty estimates and the true model error. As shown in Figure~\ref{fig:uq}, on both tasks, \textbf{Max Var} is unable to accurately predict the true model error, suggesting that uncertainty estimation used by offline model-based methods is not accurate and might be the major factor that results in its poor performance. Meanwhile, COMBO circumvents challenging uncertainty quantification problem and achieves better performances on those tasks, indicating the effectiveness and the robustness of the method.}

\subsection{Results on image-based tasks}
\vspace{-0.1cm}
To answer question (2), we evaluate COMBO on two image-based environments: the standard walker (\texttt{walker-walk}) task from the the DeepMind Control suite \citep{tassa2018deepmind} and a visual door opening environment with a Sawyer robotic arm (\texttt{sawyer-door}) as used in Section~\ref{sec:generalization_exps}.
\begin{wraptable}{r}{9cm}
\vspace{-0.4cm}
\centering
\scriptsize
\begin{tabular}{l|l|r|r|r|r|r}
\toprule
\textbf{\!\!\!Dataset\!\!} & \textbf{Environment} & \vtop{\hbox{\bf \strut COMBO}\hbox{\strut \bf (Ours)}}  & \textbf{LOMPO}& \textbf{LMBPO}& \vtop{\hbox{\bf \strut SLAC}\hbox{\strut \bf -Off}} & \textbf{CQL}\\ \midrule

\!\!\!M-R           & walker\_walk & \textbf{69.2}  & 66.9 & 59.8  & 45.1 & 15.6          \\
\!\!\!M             & walker\_walk & 57.7  & 60.2 & \textbf{61.7}  & 41.5 & 38.9          \\
\!\!\!M-E           & walker\_walk & 76.4  & \textbf{78.9} & 47.3  & 34.9 & 36.3          \\
\!\!\!expert        & walker\_walk & \textbf{61.1}  & 55.6 & 13.2  & 12.6 & 43.3          \\

\!\!\!M-E\!\!\!     & sawyer-door &\textbf{100.0\%} & \textbf{100.0\%}& 0.0\%  & 0.0\% & 0.0\% \\
\!\!\!expert\!\!\!        & sawyer-door &\textbf{96.7\%}  & 0.0\%           & 0.0\% & 0.0\%  & 0.0\% \\
\bottomrule
\end{tabular}
\vspace{-0.2cm}
\caption{\footnotesize Results for vision experiments. %
For the Walker task each number is the normalized score proposed in \cite{fu2020d4rl} of the policy at the last iteration of training, averaged over 3 random seeds. For the Sawyer task, we report success rates over the last 100 evaluation runs of training. For the dataset, M refers to medium, M-R refers to medium-replay, and M-E refers to medium expert.}
\label{tbl:vision}
\normalsize
\vspace{-0.5cm}
\end{wraptable}
For the walker task we construct 4 datasets: medium-replay (M-R), medium (M), medium-expert (M-E), and expert, similar to \citet{fu2020d4rl}, each consisting of 200 trajectories. For \texttt{sawyer-door} task we use only the medium-expert and the expert datasets, due to the sparse reward -- the agent is rewarded only when it successfully opens the door. Both environments are visulized in Figure~\ref{fig:visual} in Appendix~\ref{app:image_details}.
To extend COMBO to the image-based setting, we follow \citet{Rafailov2020LOMPO} and train a recurrent variational model using the offline data and use train COMBO in the latent space of this model.
We present results in Table~\ref{tbl:vision}. On the \texttt{walker-walk} task, COMBO performs in line with LOMPO and previous methods. On the more challenging Sawyer task, COMBO matches LOMPO and achieves 100\% success rate on the medium-expert dataset, and substantially outperforms all other methods on the narrow expert dataset, achieving an average success rate of 96.7\%, when all other model-based and model-free methods fail.

\vspace*{-5pt}
\subsection{Results on the D4RL tasks}
\label{sec:d4rl_exps}

\neurips{Finally, to answer the question (3)}, we evaluate COMBO on the OpenAI Gym~\citep{brockman2016openai} domains in the D4RL benchmark~\citep{fu2020d4rl}, which contains three environments (halfcheetah, hopper, and walker2d) and four dataset types (random, medium, medium-replay, and medium-expert). We include the results in Table~\ref{tbl:d4rl}. The numbers of BC, SAC-off, BEAR, BRAC-P and BRAC-v are taken from the D4RL paper, while the results for MOPO, \neurips{MOReL} and CQL are based on their respective papers~\citep{yu2020mopo,kumar2020conservative}. 
\neurips{COMBO achieves the best performance in 9 out of 12 settings and comparable result in 1 out of the remaining 3 settings (hopper medium-replay).}
As noted by \citet{yu2020mopo} and \citet{Rafailov2020LOMPO}, model-based offline methods are generally more performant on datasets that are collected by a wide range of policies and have diverse state-action distributions (random, medium-replay datasets)
while model-free approaches do better on datasets with narrow distributions (medium, medium-expert datasets). However, in these results, \neurips{COMBO generally performs well across dataset types compared to existing model-free and model-based approaches, suggesting that COMBO is robust to different dataset types.}

\begin{table*}[t!]
\centering
\vspace*{0.1cm}
\small
\resizebox{0.98\textwidth}{!}{\begin{tabular}{l|l|r|r|r|r|r|r|r|r|r}
\toprule
\textbf{\!\!\!Dataset type\!\!} & \textbf{Environment} & \textbf{BC} & \vtop{\hbox{\bf \strut COMBO}\hbox{\strut \bf (Ours)}}& \textbf{MOPO} & \neurips{\textbf{MOReL}} & \textbf{CQL} & \textbf{SAC-off} & \textbf{BEAR} & \textbf{BRAC-p} & \textbf{BRAC-v}\\ \midrule
\!\!\!random & halfcheetah & 2.1 & \textbf{38.8}$\pm$3.7 & 35.4 & 25.6 & 35.4 & 30.5 & 25.1 & 24.1 & 31.2 \\
\!\!\!random & hopper & 1.6 & 17.9$\pm$1.4 & 11.7 & \textbf{53.6} & 10.8 & 11.3 & 11.4 & 11.0 & 12.2 \\
\!\!\!random & walker2d & 9.8 & 7.0$\pm$3.6 & 13.6 & \textbf{37.3}  & 7.0 & 4.1 & 7.3 & -0.2 & 1.9 \\
\!\!\!medium & halfcheetah & 36.1 & \textbf{54.2}$\pm$1.5 & 42.3 & 42.1  & 44.4 & -4.3 & 41.7 & 43.8 & 46.3 \\
\!\!\!medium & hopper & 29.0 & \textbf{97.2}$\pm$2.2 & 28.0 & 95.4  & 86.6 & 0.8 & 52.1 & 32.7 & 31.1 \\
\!\!\!medium & walker2d & 6.6 & \textbf{81.9}$\pm$2.8 & 17.8 & 77.8  & 74.5 & 0.9 & 59.1 & 77.5 & 81.1 \\
\!\!\!medium-replay & halfcheetah & 38.4  & \textbf{55.1}$\pm$1.0 & 53.1 & 40.2 & 46.2 & -2.4 & 38.6 & 45.4 & 47.7 \\
\!\!\!medium-replay & hopper & 11.8 & 89.5$\pm$1.8 & 67.5 & \textbf{93.6}  & 48.6 & 3.5 & 33.7 & 0.6 & 0.6 \\
\!\!\!medium-replay & walker2d & 11.3 & \textbf{56.0}$\pm$8.6 & 39.0 & 49.8  & 32.6 & 1.9 & 19.2 & -0.3 & 0.9 \\
\!\!\!med-expert\!\!\! & halfcheetah & 35.8 & \textbf{90.0}$\pm$5.6 & 63.3 & 53.3  & 62.4 & 1.8 & 53.4 & 44.2 & 41.9 \\
\!\!\!med-expert\!\!\! & hopper & 111.9  & \textbf{111.1}$\pm$2.9 & 23.7 & 108.7  & 111.0 & 1.6 & 96.3 & 1.9 & 0.8 \\
\!\!\!med-expert\!\!\! & walker2d & 6.4 & \textbf{103.3}$\pm$5.6 & 44.6 & 95.6  & 98.7 & -0.1 & 40.1 & 76.9 & 81.6\\
\bottomrule
\end{tabular}}
\caption{\footnotesize Results for D4RL datasets. %
Each number is the normalized score proposed in \cite{fu2020d4rl} of the policy at the last iteration of training, averaged over 6 random seeds. We take results of MOPO, \neurips{MOReL} and CQL from their original papers and results of other model-free methods from \citep{fu2020d4rl}.
We include the performance of behavior cloning (\textbf{BC}) for comparison. We include the 95\%-confidence interval for COMBO. We bold the highest score across all methods.
}
\label{tbl:d4rl}
\normalsize
\vspace{-0.7cm}
\end{table*}

\section{Related Work}
\label{sec:related}

Offline RL~\citep{ernst2005tree, riedmiller2005neural, LangeGR12, levine2020offline} is the task of learning policies from a static dataset of past interactions with the environment. It has found applications in domains including robotic manipulation~\citep{ kalashnikov2018scalable, mandlekar2020iris, Rafailov2020LOMPO,singh2020cog}, NLP~\citep{jaques2019way,jaques2020human} and healthcare~\citep{shortreed2011informing, Wang2018SupervisedRL}. Similar to interactive RL, both model-free and model-based algorithms have been studied for offline RL, with explicit or implicit regularization of the learning algorithm playing a major role.

\textbf{Model-free offline RL.} Prior model-free offline RL algorithms have been designed to regularize the learned policy to be ``close`` to the behavioral policy either implicitly via regularized variants of importance sampling based algorithms~\citep{precup2001off, sutton2016emphatic, LiuSAB19, SwaminathanJ15, nachum2019algaedice}, offline actor-critic methods~\citep{siegel2020keep, peng2019advantage,kostrikov2021offline,ghasemipour2021emaq,wu2021uncertainty}, applying uncertainty quantification to the predictions of the Q-values~\citep{agarwal2020optimistic, kumar2019stabilizing, wu2019behavior, levine2020offline}, and learning conservative Q-values~\citep{kumar2020conservative,sinha2021s4rl} or explicitly measured by direct state or action constraints~\citep{fujimoto2018off,liu2020provably},
KL divergence~\citep{jaques2019way,wu2019behavior, zhou2020plas}, Wasserstein distance, MMD~\citep{kumar2019stabilizing} and auxiliary imitation loss~\citep{fujimoto2021minimalist}. 
Different from these works, COMBO uses both the offline dataset as well as model-generated data. 

\textbf{Model-based offline RL.} Model-based offline RL methods~\citep{finn2017deep, ebert2018visual, kahn2018composable, kidambi2020morel, yu2020mopo, matsushima2020deployment, argenson2020model, swazinna2020overcoming,Rafailov2020LOMPO, lee2021representation,zhan2021model} provide an alternative approach to policy learning that involves the learning of a dynamics model using techniques from supervised learning and generative modeling. Such methods however rely either on uncertainty quantification of the learned dynamics model which can be difficult for deep network models~\citep{ovadia2019can}, or on directly constraining the policy towards the behavioral policy similar to model-free algorithms~\citep{matsushima2020deployment}. In contrast, COMBO conservatively estimates the value function by penalizing it in out-of-support states generated through model rollouts. This allows COMBO to retain all benefits of model-based algorithms such as broad generalization, without the constraints of explicit policy regularization or uncertainty quantification.

\section{Conclusion}
\label{sec:concl}
In the paper, we present conservative offline model-based policy optimization (COMBO), a model-based offline RL algorithm that penalizes the Q-values evaluated on out-of-support state-action pairs. In particular, COMBO removes the need of uncertainty quantification as widely used in previous model-based offline RL works~\citep{kidambi2020morel,yu2020mopo}, which can be challenging and unreliable with deep neural networks~\citep{ovadia2019can}. Theoretically, we show that COMBO achieves \neurips{less conservative Q values} compared to prior model-free offline RL methods~\citep{kumar2020conservative} and guarantees a safe policy improvement. In our empirical study,  \neurips{COMBO achieves the best generalization performances in 3 tasks that require adaptation to unseen behaviors. Moreover, COMBO is able scale to vision-based tasks and outperforms or obtain comparable results in vision-based locomotion and robotic manipulation tasks. Finlly, on standard D4RL benchmark, COMBO generally performs well across dataset types compared to prior methods}
\iclr{Despite the advantages of COMBO, there are few challenges left such as the lack of an offline hyperparameter selection scheme that can yield a uniform hyperparameter across different datasets
and an automatically selected $f$ conditioned on the model error.} We leave them for future work.

\begin{ack}
We thank members of RAIL and IRIS for their support and feedback. This work was supported in part by ONR grants N00014-20-1-2675 and N00014-21-1-2685 as well as Intel Corporation. AK and SL are supported by the DARPA Assured Autonomy program. AR was supported by the J.P.~Morgan PhD Fellowship in AI.
\end{ack}

\bibliography{reference}

\begin{thebibliography}{69}
\providecommand{\natexlab}[1]{#1}
\providecommand{\url}[1]{\texttt{#1}}
\expandafter\ifx\csname urlstyle\endcsname\relax
  \providecommand{\doi}[1]{doi: #1}\else
  \providecommand{\doi}{doi: \begingroup \urlstyle{rm}\Url}\fi

\bibitem[Agarwal et~al.(2019)Agarwal, Jiang, and Kakade]{ajksbook}
Alekh Agarwal, Nan Jiang, and Sham~M Kakade.
\newblock Reinforcement learning: Theory and algorithms.
\newblock \emph{CS Dept., UW Seattle, Seattle, WA, USA, Tech. Rep}, 2019.

\bibitem[Agarwal et~al.(2020)Agarwal, Schuurmans, and
  Norouzi]{agarwal2020optimistic}
Rishabh Agarwal, Dale Schuurmans, and Mohammad Norouzi.
\newblock An optimistic perspective on offline reinforcement learning.
\newblock In \emph{International Conference on Machine Learning}, pages
  104--114. PMLR, 2020.

\bibitem[Argenson and Dulac-Arnold(2020)]{argenson2020model}
Arthur Argenson and Gabriel Dulac-Arnold.
\newblock Model-based offline planning.
\newblock \emph{arXiv preprint arXiv:2008.05556}, 2020.

\bibitem[Azizzadenesheli et~al.(2018)Azizzadenesheli, Brunskill, and
  Anandkumar]{Azizzadenesheli18}
Kamyar Azizzadenesheli, Emma Brunskill, and Animashree Anandkumar.
\newblock Efficient exploration through bayesian deep q-networks.
\newblock In \emph{ITA}, pages 1--9. IEEE, 2018.

\bibitem[Bertsekas and Tsitsiklis(1996)]{BertsekasBook}
Dimitri~P. Bertsekas and John~N. Tsitsiklis.
\newblock \emph{Neuro-Dynamic Programming}.
\newblock Athena Scientific, Belmont, MA, 1996.

\bibitem[Brockman et~al.(2016)Brockman, Cheung, Pettersson, Schneider,
  Schulman, Tang, and Zaremba]{brockman2016openai}
Greg Brockman, Vicki Cheung, Ludwig Pettersson, Jonas Schneider, John Schulman,
  Jie Tang, and Wojciech Zaremba.
\newblock Openai gym.
\newblock \emph{arXiv preprint arXiv:1606.01540}, 2016.

\bibitem[Clavera et~al.(2020)Clavera, Fu, and Abbeel]{clavera2020model}
Ignasi Clavera, Violet Fu, and Pieter Abbeel.
\newblock Model-augmented actor-critic: Backpropagating through paths.
\newblock \emph{arXiv preprint arXiv:2005.08068}, 2020.

\bibitem[Degris et~al.(2012)Degris, White, and Sutton]{degris2012off}
Thomas Degris, Martha White, and Richard~S Sutton.
\newblock Off-policy actor-critic.
\newblock \emph{arXiv preprint arXiv:1205.4839}, 2012.

\bibitem[Ebert et~al.(2018)Ebert, Finn, Dasari, Xie, Lee, and
  Levine]{ebert2018visual}
Frederik Ebert, Chelsea Finn, Sudeep Dasari, Annie Xie, Alex Lee, and Sergey
  Levine.
\newblock Visual foresight: Model-based deep reinforcement learning for
  vision-based robotic control.
\newblock \emph{arXiv preprint arXiv:1812.00568}, 2018.

\bibitem[Ernst et~al.(2005)Ernst, Geurts, and Wehenkel]{ernst2005tree}
Damien Ernst, Pierre Geurts, and Louis Wehenkel.
\newblock Tree-based batch mode reinforcement learning.
\newblock \emph{Journal of Machine Learning Research}, 6:\penalty0 503--556,
  2005.

\bibitem[Finn and Levine(2017)]{finn2017deep}
Chelsea Finn and Sergey Levine.
\newblock Deep visual foresight for planning robot motion.
\newblock In \emph{2017 IEEE International Conference on Robotics and
  Automation (ICRA)}, pages 2786--2793. IEEE, 2017.

\bibitem[Fu et~al.(2020)Fu, Kumar, Nachum, Tucker, and Levine]{fu2020d4rl}
Justin Fu, Aviral Kumar, Ofir Nachum, George Tucker, and Sergey Levine.
\newblock D4rl: Datasets for deep data-driven reinforcement learning, 2020.

\bibitem[Fujimoto and Gu(2021)]{fujimoto2021minimalist}
Scott Fujimoto and Shixiang~Shane Gu.
\newblock A minimalist approach to offline reinforcement learning.
\newblock \emph{arXiv preprint arXiv:2106.06860}, 2021.

\bibitem[Fujimoto et~al.(2018{\natexlab{a}})Fujimoto, Meger, and
  Precup]{fujimoto2018off}
Scott Fujimoto, David Meger, and Doina Precup.
\newblock Off-policy deep reinforcement learning without exploration.
\newblock \emph{arXiv preprint arXiv:1812.02900}, 2018{\natexlab{a}}.

\bibitem[Fujimoto et~al.(2018{\natexlab{b}})Fujimoto, Van~Hoof, and
  Meger]{fujimoto2018addressing}
Scott Fujimoto, Herke Van~Hoof, and David Meger.
\newblock Addressing function approximation error in actor-critic methods.
\newblock \emph{arXiv preprint arXiv:1802.09477}, 2018{\natexlab{b}}.

\bibitem[Ghasemipour et~al.(2021)Ghasemipour, Schuurmans, and
  Gu]{ghasemipour2021emaq}
Seyed Kamyar~Seyed Ghasemipour, Dale Schuurmans, and Shixiang~Shane Gu.
\newblock Emaq: Expected-max q-learning operator for simple yet effective
  offline and online rl.
\newblock In \emph{International Conference on Machine Learning}, pages
  3682--3691. PMLR, 2021.

\bibitem[Haarnoja et~al.(2018)Haarnoja, Zhou, Abbeel, and
  Levine]{haarnoja2018soft}
Tuomas Haarnoja, Aurick Zhou, Pieter Abbeel, and Sergey Levine.
\newblock Soft actor-critic: Off-policy maximum entropy deep reinforcement
  learning with a stochastic actor.
\newblock \emph{arXiv preprint arXiv:1801.01290}, 2018.

\bibitem[Hafner et~al.(2019)Hafner, Lillicrap, Fischer, Villegas, Ha, Lee, and
  Davidson]{Hafner2019PlanNet}
Danijar Hafner, Timothy Lillicrap, Ian Fischer, Ruben Villegas, David Ha,
  Honglak Lee, and James Davidson.
\newblock International conference on machine learning.
\newblock In \emph{International Conference on Machine Learning}, 2019.

\bibitem[Jaksch et~al.(2010)Jaksch, Ortner, and Auer]{jaksch2010near}
Thomas Jaksch, Ronald Ortner, and Peter Auer.
\newblock Near-optimal regret bounds for reinforcement learning.
\newblock \emph{Journal of Machine Learning Research}, 11\penalty0 (4), 2010.

\bibitem[Janner et~al.(2019)Janner, Fu, Zhang, and Levine]{janner2019trust}
Michael Janner, Justin Fu, Marvin Zhang, and Sergey Levine.
\newblock When to trust your model: Model-based policy optimization.
\newblock In \emph{Advances in Neural Information Processing Systems}, pages
  12498--12509, 2019.

\bibitem[Jaques et~al.(2019)Jaques, Ghandeharioun, Shen, Ferguson, Lapedriza,
  Jones, Gu, and Picard]{jaques2019way}
Natasha Jaques, Asma Ghandeharioun, Judy~Hanwen Shen, Craig Ferguson, Agata
  Lapedriza, Noah Jones, Shixiang Gu, and Rosalind Picard.
\newblock Way off-policy batch deep reinforcement learning of implicit human
  preferences in dialog.
\newblock \emph{arXiv preprint arXiv:1907.00456}, 2019.

\bibitem[Jaques et~al.(2020)Jaques, Shen, Ghandeharioun, Ferguson, Lapedriza,
  Jones, Gu, and Picard]{jaques2020human}
Natasha Jaques, Judy~Hanwen Shen, Asma Ghandeharioun, Craig Ferguson, Agata
  Lapedriza, Noah Jones, Shixiang~Shane Gu, and Rosalind Picard.
\newblock Human-centric dialog training via offline reinforcement learning.
\newblock \emph{arXiv preprint arXiv:2010.05848}, 2020.

\bibitem[Jin et~al.(2021)Jin, Yang, and Wang]{jin2021pessimism}
Ying Jin, Zhuoran Yang, and Zhaoran Wang.
\newblock Is pessimism provably efficient for offline rl?
\newblock In \emph{International Conference on Machine Learning}, pages
  5084--5096. PMLR, 2021.

\bibitem[Kahn et~al.(2018)Kahn, Villaflor, Abbeel, and
  Levine]{kahn2018composable}
Gregory Kahn, Adam Villaflor, Pieter Abbeel, and Sergey Levine.
\newblock Composable action-conditioned predictors: Flexible off-policy
  learning for robot navigation.
\newblock In \emph{Conference on Robot Learning}, pages 806--816. PMLR, 2018.

\bibitem[Kalashnikov et~al.(2018)Kalashnikov, Irpan, Pastor, Ibarz, Herzog,
  Jang, Quillen, Holly, Kalakrishnan, Vanhoucke,
  et~al.]{kalashnikov2018scalable}
Dmitry Kalashnikov, Alex Irpan, Peter Pastor, Julian Ibarz, Alexander Herzog,
  Eric Jang, Deirdre Quillen, Ethan Holly, Mrinal Kalakrishnan, Vincent
  Vanhoucke, et~al.
\newblock Scalable deep reinforcement learning for vision-based robotic
  manipulation.
\newblock In \emph{Conference on Robot Learning}, pages 651--673. PMLR, 2018.

\bibitem[Kidambi et~al.(2020)Kidambi, Rajeswaran, Netrapalli, and
  Joachims]{kidambi2020morel}
Rahul Kidambi, Aravind Rajeswaran, Praneeth Netrapalli, and Thorsten Joachims.
\newblock Morel: Model-based offline reinforcement learning.
\newblock \emph{arXiv preprint arXiv:2005.05951}, 2020.

\bibitem[Kostrikov et~al.(2021)Kostrikov, Fergus, Tompson, and
  Nachum]{kostrikov2021offline}
Ilya Kostrikov, Rob Fergus, Jonathan Tompson, and Ofir Nachum.
\newblock Offline reinforcement learning with fisher divergence critic
  regularization.
\newblock In \emph{International Conference on Machine Learning}, pages
  5774--5783. PMLR, 2021.

\bibitem[Kumar et~al.(2019)Kumar, Fu, Soh, Tucker, and
  Levine]{kumar2019stabilizing}
Aviral Kumar, Justin Fu, Matthew Soh, George Tucker, and Sergey Levine.
\newblock Stabilizing off-policy q-learning via bootstrapping error reduction.
\newblock In \emph{Advances in Neural Information Processing Systems}, pages
  11761--11771, 2019.

\bibitem[Kumar et~al.(2020)Kumar, Zhou, Tucker, and
  Levine]{kumar2020conservative}
Aviral Kumar, Aurick Zhou, George Tucker, and Sergey Levine.
\newblock Conservative q-learning for offline reinforcement learning.
\newblock \emph{arXiv preprint arXiv:2006.04779}, 2020.

\bibitem[Lange et~al.(2012)Lange, Gabel, and Riedmiller]{LangeGR12}
Sascha Lange, Thomas Gabel, and Martin~A. Riedmiller.
\newblock Batch reinforcement learning.
\newblock In \emph{Reinforcement Learning}, volume~12. Springer, 2012.

\bibitem[Laroche et~al.(2019)Laroche, Trichelair, and
  Des~Combes]{laroche2019safe}
Romain Laroche, Paul Trichelair, and Remi~Tachet Des~Combes.
\newblock Safe policy improvement with baseline bootstrapping.
\newblock In \emph{International Conference on Machine Learning}, pages
  3652--3661. PMLR, 2019.

\bibitem[Lee et~al.(2020)Lee, Nagabandi, Abbeel, and Levine]{lee2019SLAC}
Alex~X. Lee, Anusha Nagabandi, Pieter Abbeel, and Sergey Levine.
\newblock Stochastic latent actor-critic: Deep reinforcement learning with a
  latent variable model.
\newblock In \emph{Advances in Neural Information Processing Systems}, 2020.

\bibitem[Lee et~al.(2021)Lee, Lee, and Kim]{lee2021representation}
Byung-Jun Lee, Jongmin Lee, and Kee-Eung Kim.
\newblock Representation balancing offline model-based reinforcement learning.
\newblock In \emph{International Conference on Learning Representations}, 2021.
\newblock URL \url{https://openreview.net/forum?id=QpNz8r_Ri2Y}.

\bibitem[Levine et~al.(2020)Levine, Kumar, Tucker, and Fu]{levine2020offline}
Sergey Levine, Aviral Kumar, George Tucker, and Justin Fu.
\newblock Offline reinforcement learning: Tutorial, review, and perspectives on
  open problems.
\newblock \emph{arXiv preprint arXiv:2005.01643}, 2020.

\bibitem[Liu et~al.(2019)Liu, Swaminathan, Agarwal, and Brunskill]{LiuSAB19}
Yao Liu, Adith Swaminathan, Alekh Agarwal, and Emma Brunskill.
\newblock Off-policy policy gradient with state distribution correction.
\newblock \emph{CoRR}, abs/1904.08473, 2019.

\bibitem[Liu et~al.(2020)Liu, Swaminathan, Agarwal, and
  Brunskill]{liu2020provably}
Yao Liu, Adith Swaminathan, Alekh Agarwal, and Emma Brunskill.
\newblock Provably good batch reinforcement learning without great exploration.
\newblock \emph{arXiv preprint arXiv:2007.08202}, 2020.

\bibitem[Lowrey et~al.(2019)Lowrey, Rajeswaran, Kakade, Todorov, and
  Mordatch]{POLO}
Kendall Lowrey, Aravind Rajeswaran, Sham Kakade, Emanuel Todorov, and Igor
  Mordatch.
\newblock {Plan Online, Learn Offline: Efficient Learning and Exploration via
  Model-Based Control}.
\newblock In \emph{{International Conference on Learning Representations
  (ICLR)}}, 2019.

\bibitem[Mandlekar et~al.(2020)Mandlekar, Ramos, Boots, Savarese, Fei-Fei,
  Garg, and Fox]{mandlekar2020iris}
Ajay Mandlekar, Fabio Ramos, Byron Boots, Silvio Savarese, Li~Fei-Fei, Animesh
  Garg, and Dieter Fox.
\newblock Iris: Implicit reinforcement without interaction at scale for
  learning control from offline robot manipulation data.
\newblock In \emph{2020 IEEE International Conference on Robotics and
  Automation (ICRA)}, pages 4414--4420. IEEE, 2020.

\bibitem[Matsushima et~al.(2020)Matsushima, Furuta, Matsuo, Nachum, and
  Gu]{matsushima2020deployment}
Tatsuya Matsushima, Hiroki Furuta, Yutaka Matsuo, Ofir Nachum, and Shixiang Gu.
\newblock Deployment-efficient reinforcement learning via model-based offline
  optimization.
\newblock \emph{arXiv preprint arXiv:2006.03647}, 2020.

\bibitem[Munos and Szepesvari(2008)]{Munos2008FiniteTimeBF}
R{\'e}mi Munos and Csaba Szepesvari.
\newblock Finite-time bounds for fitted value iteration.
\newblock \emph{J. Mach. Learn. Res.}, 9:\penalty0 815--857, 2008.

\bibitem[Nachum et~al.(2019)Nachum, Dai, Kostrikov, Chow, Li, and
  Schuurmans]{nachum2019algaedice}
Ofir Nachum, Bo~Dai, Ilya Kostrikov, Yinlam Chow, Lihong Li, and Dale
  Schuurmans.
\newblock Algaedice: Policy gradient from arbitrary experience.
\newblock \emph{arXiv preprint arXiv:1912.02074}, 2019.

\bibitem[Osband and Van~Roy(2017)]{osband2017posterior}
Ian Osband and Benjamin Van~Roy.
\newblock Why is posterior sampling better than optimism for reinforcement
  learning?
\newblock In \emph{International Conference on Machine Learning}, pages
  2701--2710. PMLR, 2017.

\bibitem[Osband et~al.(2018)Osband, Aslanides, and Cassirer]{OsbandAC18}
Ian Osband, John Aslanides, and Albin Cassirer.
\newblock Randomized prior functions for deep reinforcement learning.
\newblock \emph{CoRR}, abs/1806.03335, 2018.

\bibitem[Ovadia et~al.(2019)Ovadia, Fertig, Ren, Nado, Sculley, Nowozin,
  Dillon, Lakshminarayanan, and Snoek]{ovadia2019can}
Yaniv Ovadia, Emily Fertig, Jie Ren, Zachary Nado, David Sculley, Sebastian
  Nowozin, Joshua~V Dillon, Balaji Lakshminarayanan, and Jasper Snoek.
\newblock Can you trust your model's uncertainty? evaluating predictive
  uncertainty under dataset shift.
\newblock \emph{arXiv preprint arXiv:1906.02530}, 2019.

\bibitem[Peng et~al.(2019)Peng, Kumar, Zhang, and Levine]{peng2019advantage}
Xue~Bin Peng, Aviral Kumar, Grace Zhang, and Sergey Levine.
\newblock Advantage-weighted regression: Simple and scalable off-policy
  reinforcement learning.
\newblock \emph{arXiv preprint arXiv:1910.00177}, 2019.

\bibitem[Petrik et~al.(2016)Petrik, Chow, and Ghavamzadeh]{petrik2016safe}
Marek Petrik, Yinlam Chow, and Mohammad Ghavamzadeh.
\newblock Safe policy improvement by minimizing robust baseline regret.
\newblock \emph{arXiv preprint arXiv:1607.03842}, 2016.

\bibitem[Precup et~al.(2001)Precup, Sutton, and Dasgupta]{precup2001off}
Doina Precup, Richard~S Sutton, and Sanjoy Dasgupta.
\newblock Off-policy temporal-difference learning with function approximation.
\newblock In \emph{ICML}, pages 417--424, 2001.

\bibitem[Rafailov et~al.(2020)Rafailov, Yu, Rajeswaran, and
  Finn]{Rafailov2020LOMPO}
Rafael Rafailov, Tianhe Yu, A.~Rajeswaran, and Chelsea Finn.
\newblock Offline reinforcement learning from images with latent space models.
\newblock \emph{ArXiv}, abs/2012.11547, 2020.

\bibitem[Rashidinejad et~al.(2021)Rashidinejad, Zhu, Ma, Jiao, and
  Russell]{rashidinejad2021bridging}
Paria Rashidinejad, Banghua Zhu, Cong Ma, Jiantao Jiao, and Stuart Russell.
\newblock Bridging offline reinforcement learning and imitation learning: A
  tale of pessimism.
\newblock \emph{arXiv preprint arXiv:2103.12021}, 2021.

\bibitem[Riedmiller(2005)]{riedmiller2005neural}
Martin Riedmiller.
\newblock Neural fitted q iteration--first experiences with a data efficient
  neural reinforcement learning method.
\newblock In \emph{European Conference on Machine Learning}, pages 317--328.
  Springer, 2005.

\bibitem[Ross and Bagnell(2012)]{RossB12}
Stephane Ross and Drew Bagnell.
\newblock Agnostic system identification for model-based reinforcement
  learning.
\newblock In \emph{ICML}, 2012.

\bibitem[Shortreed et~al.(2011)Shortreed, Laber, Lizotte, Stroup, Pineau, and
  Murphy]{shortreed2011informing}
Susan~M Shortreed, Eric Laber, Daniel~J Lizotte, T~Scott Stroup, Joelle Pineau,
  and Susan~A Murphy.
\newblock Informing sequential clinical decision-making through reinforcement
  learning: an empirical study.
\newblock \emph{Machine learning}, 84\penalty0 (1-2):\penalty0 109--136, 2011.

\bibitem[Siegel et~al.(2020)Siegel, Springenberg, Berkenkamp, Abdolmaleki,
  Neunert, Lampe, Hafner, and Riedmiller]{siegel2020keep}
Noah~Y Siegel, Jost~Tobias Springenberg, Felix Berkenkamp, Abbas Abdolmaleki,
  Michael Neunert, Thomas Lampe, Roland Hafner, and Martin Riedmiller.
\newblock Keep doing what worked: Behavioral modelling priors for offline
  reinforcement learning.
\newblock \emph{arXiv preprint arXiv:2002.08396}, 2020.

\bibitem[Singh et~al.(2020)Singh, Yu, Yang, Zhang, Kumar, and
  Levine]{singh2020cog}
Avi Singh, Albert Yu, Jonathan Yang, Jesse Zhang, Aviral Kumar, and Sergey
  Levine.
\newblock Cog: Connecting new skills to past experience with offline
  reinforcement learning.
\newblock \emph{arXiv preprint arXiv:2010.14500}, 2020.

\bibitem[Sinha and Garg(2021)]{sinha2021s4rl}
Samarth Sinha and Animesh Garg.
\newblock S4rl: Surprisingly simple self-supervision for offline reinforcement
  learning.
\newblock \emph{arXiv preprint arXiv:2103.06326}, 2021.

\bibitem[Sutton and Barto(1998)]{SuttonBook}
R.~S. Sutton and A.~G. Barto.
\newblock \emph{Reinforcement Learning: An Introduction}.
\newblock MIT Press, Cambridge, MA, 1998.

\bibitem[Sutton(1991)]{sutton1991dyna}
Richard~S Sutton.
\newblock Dyna, an integrated architecture for learning, planning, and
  reacting.
\newblock \emph{ACM Sigart Bulletin}, 2\penalty0 (4):\penalty0 160--163, 1991.

\bibitem[Sutton et~al.(2016)Sutton, Mahmood, and White]{sutton2016emphatic}
Richard~S Sutton, A~Rupam Mahmood, and Martha White.
\newblock An emphatic approach to the problem of off-policy temporal-difference
  learning.
\newblock \emph{The Journal of Machine Learning Research}, 17\penalty0
  (1):\penalty0 2603--2631, 2016.

\bibitem[Swaminathan and Joachims(2015)]{SwaminathanJ15}
Adith Swaminathan and Thorsten Joachims.
\newblock Batch learning from logged bandit feedback through counterfactual
  risk minimization.
\newblock \emph{J. Mach. Learn. Res}, 16:\penalty0 1731--1755, 2015.

\bibitem[Swazinna et~al.(2020)Swazinna, Udluft, and
  Runkler]{swazinna2020overcoming}
Phillip Swazinna, Steffen Udluft, and Thomas Runkler.
\newblock Overcoming model bias for robust offline deep reinforcement learning.
\newblock \emph{arXiv preprint arXiv:2008.05533}, 2020.

\bibitem[Tassa et~al.(2018)Tassa, Doron, Muldal, Erez, Li, Casas, Budden,
  Abdolmaleki, Merel, Lefrancq, et~al.]{tassa2018deepmind}
Yuval Tassa, Yotam Doron, Alistair Muldal, Tom Erez, Yazhe Li, Diego de~Las
  Casas, David Budden, Abbas Abdolmaleki, Josh Merel, Andrew Lefrancq, et~al.
\newblock Deepmind control suite.
\newblock \emph{arXiv preprint arXiv:1801.00690}, 2018.

\bibitem[Wang et~al.(2018)Wang, Zhang, He, and Zha]{Wang2018SupervisedRL}
L.~Wang, Wei Zhang, Xiaofeng He, and H.~Zha.
\newblock Supervised reinforcement learning with recurrent neural network for
  dynamic treatment recommendation.
\newblock \emph{Proceedings of the 24th ACM SIGKDD International Conference on
  Knowledge Discovery \& Data Mining}, 2018.

\bibitem[Wu et~al.(2019)Wu, Tucker, and Nachum]{wu2019behavior}
Yifan Wu, George Tucker, and Ofir Nachum.
\newblock Behavior regularized offline reinforcement learning.
\newblock \emph{arXiv preprint arXiv:1911.11361}, 2019.

\bibitem[Wu et~al.(2021)Wu, Zhai, Srivastava, Susskind, Zhang, Salakhutdinov,
  and Goh]{wu2021uncertainty}
Yue Wu, Shuangfei Zhai, Nitish Srivastava, Joshua Susskind, Jian Zhang, Ruslan
  Salakhutdinov, and Hanlin Goh.
\newblock Uncertainty weighted actor-critic for offline reinforcement learning.
\newblock \emph{arXiv preprint arXiv:2105.08140}, 2021.

\bibitem[Yu et~al.(2020{\natexlab{a}})Yu, Chen, Wang, Xian, Chen, Liu,
  Madhavan, and Darrell]{Yu2020BDD100KAD}
F.~Yu, H.~Chen, X.~Wang, Wenqi Xian, Yingying Chen, Fangchen Liu, V.~Madhavan,
  and Trevor Darrell.
\newblock Bdd100k: A diverse driving dataset for heterogeneous multitask
  learning.
\newblock \emph{2020 IEEE/CVF Conference on Computer Vision and Pattern
  Recognition (CVPR)}, pages 2633--2642, 2020{\natexlab{a}}.

\bibitem[Yu et~al.(2020{\natexlab{b}})Yu, Quillen, He, Julian, Hausman, Finn,
  and Levine]{yu2020meta}
Tianhe Yu, Deirdre Quillen, Zhanpeng He, Ryan Julian, Karol Hausman, Chelsea
  Finn, and Sergey Levine.
\newblock Meta-world: A benchmark and evaluation for multi-task and meta
  reinforcement learning.
\newblock In \emph{Conference on Robot Learning}, pages 1094--1100. PMLR,
  2020{\natexlab{b}}.

\bibitem[Yu et~al.(2020{\natexlab{c}})Yu, Thomas, Yu, Ermon, Zou, Levine, Finn,
  and Ma]{yu2020mopo}
Tianhe Yu, Garrett Thomas, Lantao Yu, Stefano Ermon, James Zou, Sergey Levine,
  Chelsea Finn, and Tengyu Ma.
\newblock Mopo: Model-based offline policy optimization.
\newblock \emph{arXiv preprint arXiv:2005.13239}, 2020{\natexlab{c}}.

\bibitem[Zhan et~al.(2021)Zhan, Zhu, and Xu]{zhan2021model}
Xianyuan Zhan, Xiangyu Zhu, and Haoran Xu.
\newblock Model-based offline planning with trajectory pruning.
\newblock \emph{arXiv preprint arXiv:2105.07351}, 2021.

\bibitem[Zhou et~al.(2020)Zhou, Bajracharya, and Held]{zhou2020plas}
Wenxuan Zhou, Sujay Bajracharya, and David Held.
\newblock Plas: Latent action space for offline reinforcement learning.
\newblock \emph{arXiv preprint arXiv:2011.07213}, 2020.

\end{thebibliography}
\bibliographystyle{plainnat}

\newpage
\appendix
\section{Proofs from Section~\ref{sec:theory}}
\label{app:proofs}

In this section, we provide proofs for theoretical results in Section~\ref{sec:theory}. Before the proofs, we note that all statements are proven in the case of finite state space (i.e., $|\states| < \infty$) and finite action space (i.e., $|\actions| < \infty$) we define some commonly appearing notation symbols appearing in the proof: 
\begin{itemize}
\vspace{-5pt}
    \item $P_{\mdp}$ and $r_{\mdp}$ (or $P$ and $r$ with no subscript for notational simplicity) denote the dynamics and reward function of the actual MDP $\mdp$
    \vspace{-5pt}
    \item $P_{\mdpbar}$ and $r_{\mdpbar}$ denote the dynamics and reward of the empirical MDP $\mdpbar$ generated from the transitions in the dataset
    \vspace{-5pt}
    \item $P_{\mdphat}$ and $r_{\mdphat}$ denote the dynamics and reward of the MDP induced by the learned model $\mdphat$
\end{itemize}
\vspace{-5pt}
We also assume that whenever the cardinality of a particular state or state-action pair in the offline dataset $\data$, denoted by $|\mathcal{D}(\bs, \ba)|$, appears in the denominator, we assume it is non-zero. For any non-existent $(\bs, \ba) \notin \data$, we can simply set $|\data(\bs, \ba)|$ to be a small value $< 1$, which prevents any bound from producing trivially $\infty$ values.

\subsection{A Useful Lemma and Its Proof}
\label{app:proof_lemma}

Before proving our main results, we first show that the penalty
term in equation \ref{eqn:combo_iterate} is positive in expectation. Such a positive penalty is important to combat any overestimation that may
arise as a result of using $\bellmanhat$.

\begin{lemma}[(Interpolation Lemma]
\label{thm:line_thm}
For any $f \in [0, 1]$, and any given $\rho(\bs, \ba) \in \Delta^{|\states||\actions|}$, let $d_f$ be an f-interpolation of $\rho$ and $\data$, i.e., $d_f(\bs, \ba) := f d(\bs, \ba) + (1-f) \rho(\bs, \ba)$. For a given iteration $k$ of Equation~\ref{eqn:combo_iterate}, we restate the definition of the expected penalty under $\rho(\bs, \ba)$ in Eq.~\ref{eqn:expected_penalty}: 
\begin{equation*}
 \nu(\rho, f) := \E_{\bs, \ba \sim \rho(\bs, \ba)}\left[\frac{\rho(\bs, \ba) - d(\bs, \ba)}{d_f(\bs, \ba)} \right].
\end{equation*}
Then $\nu(\rho, f)$ satisfies, (1) $\nu(\rho, f) \geq 0,~~ \forall \rho, f$, (2) $\nu(\rho, f)$ is monotonically increasing in $f$ for a fixed $\rho$, and (3) $\nu(\rho, f) = 0$ iff $\forall~ \bs, \ba, ~\rho(\bs, \ba) = d(\bs, \ba) \text{~or~} f = 0$. 
\end{lemma}
\begin{proof}
To prove this lemma, we use algebraic manipulation on the expression for quantity $\nu(\rho, f)$ and show that it is indeed positive and monotonically increasing in $f \in [0, 1]$.
\begin{align}
    \nu(\rho, f) &= \sum_{\bs, \ba} \rho(\bs, \ba) \left(\frac{\rho(\bs, \ba) - d(\bs, \ba)}{f d(\bs, \ba) + (1 - f) \rho(\bs, \ba)}\right)\nonumber \\
    &= \sum_{\bs, \ba} \rho(\bs, \ba) \left(\frac{\rho(\bs, \ba) - d(\bs, \ba)}{\rho(\bs, \ba) + f ( d(\bs, \ba) - \rho(\bs, \ba))}\right)\\
    \implies \frac{d \nu(\rho, f)}{d f} &= \sum_{\bs, \ba} \rho(\bs, \ba) \left(\rho(\bs, \ba) - d(\bs, \ba)\right)^2 \cdot \left(\frac{1}{(\rho(\bs, \ba) + f ( d(\bs, \ba) - \rho(\bs, \ba))}\right)^2 \geq 0\nonumber\\
    &~~~\forall f \in [0, 1].
\end{align}
Since the derivative of $\nu(\rho, f)$ with respect to $f$ is always positive, it is an increasing function of $f$ for a fixed $\rho$, and this proves the second part (2) of the Lemma. Using this property, we can show the part (1) of the Lemma as follows:
\begin{align}
    \forall f \in (0, 1],~ \nu(\rho, f) \geq \nu(\rho, 0) = \sum_{\bs, \ba} \rho(\bs, \ba) \frac{\rho(\bs, \ba) - d(\bs, \ba)}{\rho(\bs, \ba)} &= \sum_{\bs, \ba} \left( \rho(\bs, \ba) - d(\bs, \ba) \right)\nonumber\\
    &= 1 - 1 = 0.
\end{align}
Finally, to prove the third part (3) of this Lemma, note that when $f = 0$, $\nu(\rho, f) = 0$ (as shown above), and similarly by setting $\rho(\bs, \ba) = d(\bs, \ba)$ note that we obtain $\nu(\rho, f) = 0$. To prove the only if side of (3), assume that $f \neq 0$ and $\rho(\bs, \ba) \neq d(\bs, \ba)$ and we will show that in this case $\nu(\rho,f) \neq 0$. When $d(\bs, \ba) \neq \rho(\bs, \ba)$, the derivative $\frac{d \nu(\rho,f)}{d f} > 0$ (i.e., strictly positive) and hence the function $\nu(\rho, f)$ is a strictly increasing function of $f$. Thus, in this case, $\nu(\rho, f) > 0 = \nu(\rho, 0)~ \forall f > 0$. Thus we have shown that if $\rho(\bs, \ba) \neq d(\bs, \ba)$ and $f > 0$, $\nu(\rho, f) \neq 0$, which completes our proof for the only if side of (3). 
\end{proof}

\subsection{Proof of Proposition~\ref{thm:lower_bound}}
\label{app:proof_lower_bound}
Before proving this proposition, we provide a bound on the Bellman backup in the empirical MDP, $\bellman_{\mdpbar}$. To do so, we formally define the standard concentration properties of the reward and transition dynamics in the empirical MDP, $\mdpbar$, that we assume so as to prove Proposition~\ref{thm:line_thm}. Following prior work~\citep{osband2017posterior,jaksch2010near,kumar2020conservative}, we assume:
\begin{assumption}
\label{assumption:conc}
    $\forall~ \bs, \ba \in \mdp$, the following relationships hold with high probability, $\geq 1 - \delta$
    \begin{equation*}
        |r_{\mdpbar}(\bs, \ba) - r(\bs, \ba)| \leq \frac{C_{r, \delta}}{\sqrt{|\mathcal{D}(\bs, \ba)|}}, ~~~ ||P_{\mdpbar}(\bs'|\bs, \ba) - P(\bs'|\bs, \ba)||_{1} \leq \frac{C_{P, \delta}}{\sqrt{|\mathcal{D}(\bs, \ba)|}}.
    \end{equation*}
\end{assumption}
Under this assumption and assuming that the reward function in the MDP, $r(\bs, \ba)$ is bounded, as $|r(\bs, \ba)| \leq R_{\max}$, we can bound the difference between the empirical Bellman operator, $\bellman_{\mdpbar}$ and the actual MDP, $\bellman_\mdp$,
\begin{align*}
    \left\vert\left({\bellman_{\mdpbar}}^\policy \hat{Q}^k \right) - \left({\bellman}^\policy_\mdp \hat{Q}^k \right)\right\vert &= \left\vert\left(r_{\mdpbar}(\bs, \ba) - r_\mdp(\bs, \ba)\right)\right.\\
    &\left.+ \gamma \sum_{\bs'} \left({P}_{\mdpbar}(\bs'|\bs, \ba) - P_\mdp(\bs'|\bs,\ba)\right) \E_{\policy(\ba'|\bs')}\left[\hat{Q}^k(\bs' , \ba')\right]\right\vert\\
    &\leq \left\vert r_{\mdpbar}(\bs, \ba) - r_\mdp(\bs, \ba)\right\vert\\
    &+ \gamma \left\vert \sum_{\bs'} \left({P}_{\mdpbar}(\bs'|\bs, \ba) - P_\mdp(\bs'|\bs,\ba)\right) \E_{\policy(\ba'|\bs')}\left[\hat{Q}^k(\bs' , \ba')\right]\right\vert\\
    &\leq \frac{C_{r, \delta} + \gamma C_{P, \delta} 2R_{\max} / (1 - \gamma)}{\sqrt{|\mathcal{D}(\bs, \ba)|}}. 
\end{align*}
Thus the overestimation due to sampling error in the empirical MDP, $\mdpbar$ is bounded as a function of a bigger constant, $C_{r, P, \delta}$ that can be expressed as a function of $C_{r, \delta}$ and $C_{P, \delta}$, and depends on $\delta$ via a $\sqrt{\log (1/\delta)}$ dependency. For the purposes of proving Proposition~\ref{thm:Q_bound}, we assume that:
\begin{equation}
\label{eqn:sampling_error}
    \forall \bs, \ba, ~~\left\vert\left({\bellman_{\mdpbar}}^\policy \hat{Q}^k \right) - \left({\bellman}^\policy_\mdp \hat{Q}^k \right)\right\vert  \leq \frac{C_{r, T, \delta} R_{\max}}{(1 - \gamma) \sqrt{|\mathcal{D}(\bs, \ba)|}}.
\end{equation}

Next, we provide a bound on the error between the bellman backup induced by the learned dynamics model and the learned reward, $\bellman_{\mdphat}$, and the actual Bellman backup, $\bellman_{\mdp}$. To do so, we note that:
\begin{align}
    \left\vert\left({\bellman_{\mdphat}}^\policy \hat{Q}^k \right) - \left({\bellman}^\policy_\mdp \hat{Q}^k \right)\right\vert &= \left\vert\left(r_{\mdphat}(\bs, \ba) - r_\mdp(\bs, \ba)\right)\right.\\
    &\left.+ \gamma \sum_{\bs'} \left({P}_{\mdphat}(\bs'|\bs, \ba) - P_\mdp(\bs'|\bs,\ba)\right) \E_{\policy(\ba'|\bs')}\left[\hat{Q}^k(\bs' , \ba')\right]\right\vert \nonumber\\ 
    &\leq |r_{\mdphat}(\bs, \ba) - r_\mdp(\bs, \ba)| + \gamma \frac{2 R_{\max}}{1 - \gamma} D(P, P_{\mdphat}),
    \label{eqn:model_error} 
\end{align}
where $D(P, P_{\mdphat})$ is the total-variation divergence between the learned dynamics model and the actual MDP. Now, we show that the asymptotic Q-function learned by COMBO lower-bounds the actual Q-function of any
policy $\pi$ with high probability for a large enough $\beta \geq 0$. We will use Equations~\ref{eqn:sampling_error} and \ref{eqn:model_error} to prove such a result.

\begin{proposition}[Asymptotic lower-bound]
\label{thm:Q_bound}
Let $P^\pi$ denote the Hadamard product of the dynamics $P$ and a given policy $\pi$ in the actual MDP and let $S^\pi := (I - \gamma P^\pi)^{-1}$. Let $D$ denote the total-variation divergence between two probability distributions. For any $\pi(\ba|\bs)$, the Q-function obtained by recursively applying Equation~\ref{eqn:combo_iterate}, with $\hat{{\bellman}}^\pi = f \bellman_{\mdpbar}^\pi + (1 - f) \bellman_{\mdphat}^\pi$, with probability at least $1 - \delta$, results in $\hat{Q}^\pi$ that satisfies:
\begin{align*}
    \forall \bs, \ba,~ \hat{Q}^\pi(\bs, \ba) \leq  Q^\pi(\bs, \ba) &- \beta \cdot \left[ S^\pi \left[ \frac{\rho - d}{d_f} \right] \right](\bs, \ba) + f \left[ S^\pi \left[ \frac{C_{r, T, \delta} R_{\max}}{(1 - \gamma) \sqrt{|\data|}} \right] \right](\bs, \ba)\\
    +&~ (1 - f) \left[ S^\pi \left[ |r - r_{\mdphat}| + \frac{ 2 \gamma  R_{\max}}{1 - \gamma} D(P, P_{\mdphat}) \right]  \right]\!\! (\bs, \ba).
\end{align*}
\end{proposition}
\begin{proof}
We first note that the Bellman backup $\hat{\bellman}^\pi$ induces the following Q-function iterates as per Equation~\ref{eqn:combo_iterate},
\begin{align*}
    \hat{Q}^{k+1}(\bs, \ba) &= \left(\hat{\bellman}^\pi \hat{Q}^k\right)(\bs, \ba) - \beta \frac{\rho(\bs, \ba) - d(\bs, \ba)}{d_f(\bs, \ba)}\\
    &=  f \left(\bellman^\pi_{\mdpbar} \hat{Q}^k \right) (\bs, \ba) + (1 - f) \left(\bellman^\pi_{\mdphat} \hat{Q}^k \right) (\bs, \ba) - \beta \frac{\rho(\bs, \ba) - d(\bs, \ba)}{d_f(\bs, \ba)}\\
    &= \left(\bellman^\pi \hat{Q}^k\right)(\bs, \ba) - \beta \frac{\rho(\bs, \ba) - d(\bs, \ba)}{d_f(\bs, \ba)} + (1 - f) \left({\bellman_{\mdphat}}^\policy \hat{Q}^k - {\bellman}^\policy \hat{Q}^k \right)(\bs, \ba)\\
    &+ f  \left({\bellman_{\mdpbar}}^\policy \hat{Q}^k - {\bellman}^\policy \hat{Q}^k \right)(\bs, \ba)\\
   \forall \bs, \ba,~ \hat{Q}^{k+1} &\leq \left(\bellman^\pi \hat{Q}^k\right) - \beta \frac{\rho - d}{d_f} + (1 - f) \left[|r_{\mdphat} - r_\mdp| + \frac{2 \gamma R_{\max}}{1 - \gamma} D(P, P_{\mdphat}) \right] + f \frac{C_{r, T, \delta} R_{\max}}{(1 - \gamma) \sqrt{|\data|}} 
\end{align*}
Since the RHS upper bounds the Q-function pointwise for each $(\bs, \ba)$, the fixed point of the Bellman iteration process will be pointwise smaller than the fixed point of the Q-function found by solving for the RHS via equality. Thus, we get that
\begin{align*}
    \hat{Q}^\pi(\bs, \ba) &\leq \underbrace{ S^\pi r_{\mdp}}_{= Q^\pi(\bs, \ba)} -\beta \left[ S^\pi \left[ \frac{\rho - d}{d_f} \right] \right](\bs, \ba) +~ f \left[ S^\pi \left[ \frac{C_{r, T, \delta} R_{\max}}{(1 - \gamma) \sqrt{|\data|}} \right] \right](\bs, \ba)\\
    &+~ (1 - f) \left[ S^\pi \left[ |r - r_{\mdphat}| + \frac{ 2 \gamma  R_{\max}}{1 - \gamma} D(P, P_{\mdphat}) \right]  \right]\!\! (\bs, \ba),  
\end{align*}
which completes the proof of this proposition.
\end{proof}

Next, we use the result and proof technique from Proposition~\ref{thm:Q_bound} to prove Corollary~\ref{thm:lower_bound}, that in expectation under the initial state-distribution, the expected Q-value is indeed a lower-bound. 

\begin{corollary}[Corollary~\ref{thm:lower_bound} restated]
For a sufficiently large $\beta$, we have a lower-bound that
$\E_{\bs \sim \mu_0, \ba \sim \policy(\cdot|\bs)}[\hat{Q}^\pi(\bs, \ba)] \leq \E_{\bs \sim \mu_0, \ba \sim \policy(\cdot|\bs)}[Q^\pi(\bs, \ba)]$, 
where $\mu_0(\bs)$ is the initial state distribution. 
Furthermore, when $\epsilon_{\text{s}}$ is small, such as in the large sample regime; or when the model bias $\epsilon_{\text{m}}$ is small, a small $\beta$ is sufficient along with an appropriate choice of $f$.
\end{corollary}

\begin{proof}
To prove this corollary, we note a slightly different variant of Proposition~\ref{thm:Q_bound}. To observe this, we will deviate from the proof of Proposition~\ref{thm:Q_bound} slightly and will aim to express the inequality using $\bellman_{\mdphat}$, the Bellman operator defined by the learned model and the reward function. Denoting $(I - \gamma P_{\mdphat})^{-1}$ as $S_{\mdphat}^\pi$, doing this will intuitively allow us to obtain $\beta \left(\mu(\bs) \policy(\ba|\bs)\right)^T \left(S_{\mdphat}^\pi \left[\frac{\rho - d}{d_f} \right]\right)(\bs, \ba)$ as the conservative penalty which can be controlled by choosing $\beta$ appropriately so as to nullify the potential overestimation caused due to other terms. Formally,
\begin{align*}
    \hat{Q}^{k+1}(\bs, \ba) &= \left(\hat{\bellman}^\pi \hat{Q}^k\right)(\bs, \ba) - \beta \frac{\rho(\bs, \ba) - d(\bs, \ba)}{d_f(\bs, \ba)} = \left(\bellman^\pi_{\mdphat} \hat{Q}^k \right)(\bs, \ba) -  \beta \frac{\rho(\bs, \ba) - d(\bs, \ba)}{d_f(\bs, \ba)}\\
    &+ f \underbrace{\left(\bellman^\pi_{\mdpbar} - \bellman^\pi_{\mdphat} \hat{Q}^k \right)(\bs, \ba)}_{:= \Delta(\bs, \ba)}
\end{align*}
By controlling $\Delta(\bs, \ba)$ using the pointwise triangle inequality:
\begin{equation}
    \forall \bs, \ba, ~\left\vert \bellman^\pi_{\mdpbar} \hat{Q}^k - \bellman^\pi_{\mdphat} \hat{Q}^k \right\vert \leq \left\vert \bellman^\pi \hat{Q}^k - \bellman^\pi_{\mdphat} \hat{Q}^k \right\vert + \left\vert \bellman^\pi_{\mdpbar} \hat{Q}^k - \bellman^\pi \hat{Q}^k \right\vert,
\end{equation}
and then iterating the backup $\bellman^\pi_{\mdphat}$ to its fixed point and finally noting that $\rho(\bs, \ba) = \left((\mu \cdot \pi)^T S^\pi_{\mdphat}\right)(\bs, \ba)$, we obtain:
\begin{equation}
    \E_{\mu, \pi}[\hat{Q}^\pi(\bs, \ba)] \leq \E_{\mu, \pi}[Q^\pi_{\mdphat}(\bs, \ba)] - \beta~ \E_{\rho(\bs, \ba)}\left[\frac{\rho(\bs, \ba) - d(\bs, \ba)}{d_f(\bs, \ba)}\right] + \mathrm{terms~ independent~ of~} \beta.
\end{equation}
The terms marked as ``terms independent of $\beta$'' correspond to the additional positive error terms obtained by iterating $\left\vert \bellman^\pi \hat{Q}^k - \bellman^\pi_{\mdphat} \hat{Q}^k \right\vert$ and $\left\vert \bellman^\pi_{\mdpbar} \hat{Q}^k - \bellman^\pi \hat{Q}^k \right\vert$, which can be bounded similar to the proof of Proposition~\ref{thm:Q_bound} above. Now by replacing the model Q-function, $\E_{\mu, \pi}[Q^\pi_{\mdphat}(\bs, \ba)]$ with the actual Q-function, $\E_{\mu, \pi}[Q^\pi(\bs, \ba)]$ and adding an error term corresponding to model error to the bound, we obtain that:
\begin{equation}
\label{eqn:lower_bound_eqn}
    \E_{\mu, \pi}[\hat{Q}^\pi(\bs, \ba)] \leq \E_{\mu, \pi}[Q^\pi(\bs, \ba)] + \mathrm{terms~ independent~ of~} \beta - \beta~ \underbrace{\E_{\rho(\bs, \ba)}\left[\frac{\rho(\bs, \ba) - d(\bs, \ba)}{d_f(\bs, \ba)}\right]}_{= \nu(\rho, f) > 0}.
\end{equation}
Hence, by choosing $\beta$ large enough, we obtain the desired lower bound guarantee. 
\end{proof}

\begin{remark}[\underline{\textbf{COMBO does not underestimate at every $\bs \in \mathcal{D}$ unlike CQL.}}]
\label{remak:tighter_lower_bound}
Before concluding this section, we discuss how the bound obtained by COMBO (Equation~\ref{eqn:lower_bound_eqn}) is tighter than CQL. CQL learns a Q-function such that the value of the policy under the resulting Q-function lower-bounds the true value function at each state $\bs \in \mathcal{D}$ individually (in the absence of no sampling error), i.e., $\forall \bs \in \mathcal{D}, \hat{V}^\pi_{\text{CQL}}(\bs) \leq V^\pi(\bs)$, whereas the bound in COMBO is only valid in expectation of the value function over the initial state distribution, i.e., $\E_{\bs \sim \mu_0(\bs)}[\hat{V}^\pi_{\text{COMBO}}(\bs)] \leq \E_{\bs \sim \mu_0(\bs)}[V^\pi(\bs)]$, and the value function at a given state may not be a lower-bound. For instance, COMBO can overestimate the value of a state more frequent in the dataset distribution $d(\bs, \ba)$ but not so frequent in the $\rho(\bs, \ba)$ marginal distribution of the policy under the learned model $\mdphat$. To see this more formally, note that the expected penalty added in the effective Bellman backup performed by COMBO (Equation~\ref{eqn:combo_iterate}), in expectation under the dataset distribution $d(\bs, \ba)$, $\widetilde{\nu}(\rho, d, f)$ is actually \textbf{\textit{negative}}:
\begin{align*}
    \widetilde{\nu}(\rho, d, f) = \sum_{\bs, \ba} d(\bs, \ba) \frac{\rho(\bs, \ba) - d(\bs, \ba)}{d_f(\bs, \ba)} = - \sum_{\bs, \ba} d(\bs, \ba) \frac{d(\bs, \ba) - \rho(\bs, \ba)}{f d(\bs, \ba) + (1 - f) \rho(\bs, \ba)} < 0,
\end{align*}
where the final inequality follows via a direct application of the proof of Lemma~\ref{thm:line_thm}. Thus, COMBO actually \emph{overestimates} the values at atleast some states (in the dataset) unlike CQL.   
\end{remark}

\subsection{Proof of Proposition~\ref{prop:less_conservative}}
\label{app:proof_less_conservative}

In this section, we will provide a proof for Proposition~\ref{prop:less_conservative}, and show that the COMBO can be less conservative in terms of the estimated value. To recall, let $\Delta^\pi_\text{COMBO} := \E_{\bs, \ba \sim d_{\mdpbar}(\bs), \pi(\ba|\bs)}\left[\hat{Q}^\pi(\bs, \ba \right]$ and let $\Delta^\pi_\text{CQL} := \E_{\bs, \ba \sim d_{\mdpbar}, \pi(\ba|\bs)} \left[\hat{Q}^\pi_\text{CQL}(\bs, \ba) \right]$. From \citet{kumar2020conservative}, we obtain that $\hat{Q}^\pi_\text{CQL}(\bs, \ba) := Q^\pi(\bs, \ba) - \beta \frac{\pi(\ba|\bs) - \pi_\beta(\ba|\bs)}{\pi_\beta(\ba|\bs)}$. We shall derive the condition for the real data fraction $f=1$ for COMBO, thus making sure that $d_f(\bs) = d^{\pi_\beta}(\bs)$. To derive the condition when $\Delta^\pi_\text{COMBO} \geq \Delta^\pi_\text{CQL}$, we note the following simplifications:
\begin{align}
    & \Delta^\pi_\text{COMBO} \geq \Delta^\pi_\text{CQL} \\
    \implies & \sum_{\bs, \ba} d_{\mdpbar}(\bs) \pi(\ba|\bs) \hat{Q}^\pi(\bs, \ba) \geq \sum_{\bs, \ba} d_{\mdpbar}(\bs) \pi(\ba|\bs) \hat{Q}^\pi_\text{CQL}(\bs, \ba) \\
    \label{eqn:cql_vs_combo_terms}
    \implies & \beta \sum_{\bs, \ba} d_{\mdpbar}(\bs)\pi(\ba|\bs) \left( \frac{\rho(\bs, \ba) - d^{\pi_\beta}(\bs) \pi_\beta(\ba|\bs)}{d^{\pi_\beta}(\bs) \pi_\beta(\ba|\bs)} \right) \leq \beta \sum_{\bs, \ba} d_{\mdpbar}(\bs)\pi(\ba|\bs) \left(\frac{\pi(\ba|\bs) - \pi_\beta(\ba|\bs)}{\pi_\beta(\ba|\bs)} \right).
\end{align}
Now, in the expression on the left-hand side, we add and subtract $d^{\pi_\beta}(\bs) \pi(\ba|\bs)$ from the numerator inside the paranthesis.
\begin{align}
    & \sum_{\bs, \ba} d_{\mdpbar}(\bs, \ba) \left( \frac{\rho(\bs, \ba) - d^{\pi_\beta}(\bs) \pi_\beta(\ba|\bs)}{d^{\pi_\beta}(\bs) \pi_\beta(\ba|\bs)} \right)\\
    &= \sum_{\bs, \ba} d_{\mdpbar}(\bs, \ba) \left( \frac{\rho(\bs, \ba) - d^{\pi_\beta}(\bs) \pi(\ba|\bs) + d^{\pi_\beta}(\bs) \pi(\ba|\bs) - d^{\pi_\beta}(\bs) \pi_\beta(\ba|\bs)}{d^{\pi_\beta}(\bs) \pi_\beta(\ba|\bs)} \right)\\
    &= \underbrace{\sum_{\bs, \ba} d_{\mdpbar}(\bs, \ba) \frac{\pi(\ba|\bs) - \pi_\beta(\ba|\bs)}{\pi_\beta(\ba|\bs)}}_{(1)} + \sum_{\bs, \ba} d_{\mdpbar}(\bs, \ba) \cdot \frac{\rho(\bs) - d^{\pi_\beta}(\bs)}{d^{\pi_\beta}(\bs)} \cdot \frac{\pi(\ba|\bs)}{\pi_\beta(\ba|\bs)}
\end{align}
The term marked $(1)$ is identical to the CQL term that appears on the right in Equation~\ref{eqn:cql_vs_combo_terms}. Thus the inequality in Equation~\ref{eqn:cql_vs_combo_terms} is satisfied when the second term above is negative. To show this, first note that $d^{\pi_\beta}(\bs) = d_{\mdpbar}(\bs)$ which results in a cancellation. Finally, re-arranging the second term into expectations gives us the desired result. An analogous condition can be derived when $f \neq 1$, but we omit that derivation as it will be hard to interpret terms appear in the final inequality.

\subsection{Proof of Proposition~\ref{thm:policy_improvement}}
\label{app:proof_policy_improvement}

To prove the policy improvement result in Proposition~\ref{thm:policy_improvement}, we first observe that using Equation~\ref{eqn:combo_iterate} for Bellman backups amounts to finding a policy that maximizes the return of the policy in the a modified ``f-interpolant'' MDP which admits the Bellman backup $\bellmanhat^\pi$, and is induced by a linear interpolation of backups in the empirical MDP $\mdpbar$ and the MDP induced by a dynamics model $\mdphat$ and the return of a policy $\pi$ in this effective f-interpolant MDP is denoted by $J(\mdpbar, \mdphat, f, \pi)$. Alongside this, the return is penalized by the conservative penalty where $\rho^\pi$ denotes the marginal state-action distribution of policy $\pi$ in the learned model $\mdphat$. 
\begin{equation}
    \hat{J}(f, \pi) = J(\mdpbar, \mdphat, f, \pi)  - \beta \frac{\nu(\rho^\pi, f)}{1 - \gamma}.
\label{eqn:penalized_objective}
\end{equation}
We will require bounds on the return of a policy $\pi$ in this f-interpolant MDP, $J(\mdpbar, \mdphat, f, \pi)$, which we first prove separately as Lemma~\ref{lemma:interpolant_regular_bound} below and then move to the proof of Proposition~\ref{thm:policy_improvement}.

\begin{lemma}[Bound on return in f-interpolant MDP]
\label{lemma:interpolant_regular_bound}
For any two MDPs, $\mdp_1$ and $\mdp_2$, with the same state-space, action-space and discount factor, and for a given fraction $f \in [0, 1]$, define the f-interpolant MDP $\mdp_f$ as the MDP on the same state-space, action-space and with the same discount as the MDP with dynamics: $P_{\mdp_f} := f P_{\mdp_1} + (1 - f) P_{\mdp_2}$ and reward function: $r_{\mdp_f} := f r_{\mdp_1} + (1 - f) r_{\mdp_2}$. Then, given any auxiliary MDP, $\mdp$, the return of any policy $\pi$ in $\mdp_f$, $J(\pi, \mdp_f)$, also denoted by $J(\mdp_1, \mdp_2, f, \pi)$, lies in the interval:
\begin{equation*}
    \big[ J(\pi, \mdp) - \alpha,~~ J(\pi, \mdp)+ \alpha \big], \text{~~~~~~~~~~~~where~} \alpha \text{~is given by:~}
\end{equation*}
\begin{align}
    \alpha &= \frac{2 \gamma (1 - f)}{(1 - \gamma)^2} R_{\max} D \left(P_{\mdp_2}, P_{\mdp}\right) + \frac{\gamma f}{1 - \gamma} \left\vert \E_{d^\pi_{\mdp} \pi} \left[ \left(P^\pi_{\mdp} - P^\pi_{\mdp_1}\right) Q^\pi_{\mdp} \right]\right\vert  \nonumber\\
   & + \frac{f}{1 - \gamma} \E_{\bs, \ba \sim d^\pi_{\mdp} \pi}[|r_{\mdp_1}(\bs, \ba) - r_{\mdp}(\bs, \ba)|] + \frac{1 - f}{1 - \gamma} \E_{\bs, \ba \sim d^\pi_{\mdp} \pi}[|r_{\mdp_2}(\bs, \ba) - r_{\mdp}(\bs, \ba)|].  \label{eqn:alpha_expr}
\end{align}
\end{lemma}
\begin{proof}
To prove this lemma, we note two general inequalities. First, note that for a fixed transition dynamics, say $P$, the return decomposes linearly in the components of the reward as the expected return is linear in the reward function:
\begin{equation*}
    J(P, r_{\mdp_f}) = J(P, f r_{\mdp_1} + (1 - f) r_{\mdp_2}) = f J (P, r_{\mdp_1}) + (1 - f) J(P, r_{\mdp_2}).  
\end{equation*}
As a result, we can bound $J(P, r_{\mdp_f})$ using $J(P, r)$ for a new reward function $r$ of the auxiliary MDP, $\mdp$, as follows
\begin{align*}
     J(P, r_{\mdp_f}) &= J(P, f r_{\mdp_1} + (1 - f) r_{\mdp_2}) = J (P, r + f (r_{\mdp_1} - r) + (1 -f) (r_{\mdp_2} - r)\\
     &= J(P, r) + f J(P, r_{\mdp_1} - r) + (1 - f) J(P, r_{\mdp_2} - r)\\
     &= J(P, r) + \frac{f}{1 - \gamma} \E_{\bs, \ba \sim d^\pi_{\mdp}(\bs) \pi(\ba|\bs)}\left[ r_{\mdp_1}(\bs, \ba) - r(\bs, \ba) \right]\\
     &+ \frac{1 - f}{1 - \gamma} \E_{\bs, \ba \sim d^\pi_{\mdp}(\bs) \pi(\ba|\bs)} \left[ r_{\mdp_2}(\bs, \ba) - r(\bs, \ba) \right].
\end{align*}
Second, note that for a given reward function, $r$, but a linear combination of dynamics, the following bound holds:
\begin{align*}
    J(P_{\mdp_f}, r) &= J(f P_{\mdp_1} + (1 - f) P_{\mdp_2}, r)\\
    &= J ( P_{\mdp} +  f( P_{\mdp_1} - P_{\mdp}) + (1 - f) (P_{\mdp_2} - P_{\mdp}), r)\\ 
    &= J (P_{\mdp}, r) - \frac{\gamma (1 - f)}{1 - \gamma} \E_{\bs, \ba \sim d^\pi_{\mdp}(\bs) \pi(\ba|\bs)} \left[ \left(P^\pi_{\mdp_2} - P^\pi_{\mdp}\right) Q^\pi_{\mdp}  \right]\\
    &- \frac{\gamma f}{1 - \gamma} \E_{\bs, \ba \sim d^\pi_{\mdp}(\bs) \pi(\ba|\bs)} \left[ \left(P^\pi_{\mdp} - P^\pi_{\mdp_1}\right) Q^\pi_{\mdp}  \right]\\
    &\in \left[ J( P_{\mdp}, r) ~\pm~ \left(\frac{\gamma f}{(1 - \gamma)} \left\vert \E_{\bs, \ba \sim d^\pi_{\mdp}(\bs) \pi(\ba|\bs)}\left[ \left(P^\pi_{\mdp} - P^\pi_{\mdp_1}\right) Q^\pi_{\mdp} \right] \right\vert\right.\right.\\
    &\left.\left.+ \frac{2 \gamma (1 -f) R_{\max}}{(1 - \gamma)^2} D(P_{\mdp_2}, P_{\mdp}) \right) \right].
\end{align*}
To observe the third equality, we utilize the result on the difference between returns of a policy $\pi$ on two different MDPs, $P_{\mdp_1}$ and $P_{\mdp_f}$ from \citet{ajksbook} (Chapter 2, Lemma 2.2, Simulation Lemma), and additionally incorporate the auxiliary MDP $\mdp$ in the expression via addition and subtraction in the previous (second) step. In the fourth step, we finally bound one term that corresponds to the learned model via the total-variation divergence $D(P_{\mdp_2}, P_{\mdp})$ and the other term corresponding to the empirical MDP $\mdpbar$ is left in its expectation form to be bounded later. 

Using the above bounds on return for reward-mixtures and dynamics-mixtures, proving this lemma is straightforward:
\begin{align*}
    & J(\mdp_1, \mdp_2, f, \pi) := J(P_{\mdp_f}, f r_{\mdp_1} + (1 - f) r_{\mdp_2}) = J(f P_{\mdp_1} + (1 -f) P_{\mdp_2}, r_{\mdp_f})\\
    &\in \left[ J(P_{\mdp_f}, r_{\mdp}) ~\pm\right.\\
    &\left.~ \underbrace{\left(\frac{f}{1 - \gamma} \E_{\bs, \ba \sim d^\pi_{\mdp} \pi}[|r_{\mdp_1}(\bs, \ba) - r_{\mdp}(\bs, \ba)|] + \frac{1 - f}{1 - \gamma} \E_{\bs, \ba \sim d^\pi_{\mdp} \pi}[|r_{\mdp_2}(\bs, \ba) - r_{\mdp}(\bs, \ba)|] \right)}_{:= \Delta_R} \right],
\end{align*}
where the second step holds via linear decomposition of the return of $\pi$ in $\mdp_f$ with respect to the reward interpolation, and bounding the terms that appear in the reward difference. For convenience, we refer to these offset terms due to the reward as $\Delta_R$. For the final part of this proof, we bound $J(P_{\mdp_f}, r_{\mdp})$ in terms of the return on the actual MDP, $J(P_{\mdp}, r_{\mdp})$, using the inequality proved above that provides intervals for mixture dynamics but a fixed reward function. Thus, the overall bound is given by $J(\pi, \mdp_f) \in [J(\pi, \mdp) - \alpha, J(\pi, \mdp) + \alpha]$, where $\alpha$ is given by:
\begin{align}
\label{eqn:alpha_expr_repeat}
    \alpha = \frac{2 \gamma (1 - f)}{(1 - \gamma)^2} & R_{\max} D \left(P_{\mdp_2}, P_{\mdp}\right) + \frac{\gamma f}{1 - \gamma} \left\vert \E_{d^\pi_{\mdp} \pi} \left[ \left(P^\pi_{\mdp} - P^\pi_{\mdp_1}\right) Q^\pi_{\mdp} \right]\right\vert + \Delta_R.
\end{align}
This concludes the proof of this lemma.
\end{proof}

Finally, we prove Theorem~\ref{thm:policy_improvement} that shows how policy optimization with respect to $\hat{J}(f, \pi)$ affects the performance in the actual MD by using Equation~\ref{eqn:penalized_objective} and building on the  analysis of pure model-free algorithms from \citet{kumar2020conservative}. We restate a more complete statement of the theorem below and present the constants at the end of the proof. 

\begin{theorem}[Formal version of Proposition~\ref{thm:policy_improvement}]
Let $\hat{\pi}_{\text{out}}(\ba|\bs)$ be the policy obtained by COMBO. Assume $\nu(\rho^{\pi_{\text{out}}}, f) - \nu(\rho^\beta, f) \geq C$ for some constant $C > 0$.
Then, the policy ${\pi}_{\text{out}}(\ba|\bs)$ is a $\zeta$-safe policy improvement over ${\behavior}$ in the actual MDP $\mdp$, i.e., $J({\pi}_{\text{out}}, \mdp) \geq J({\behavior}, \mdp) - \zeta$, with probability at least $1 - \delta$, where $\zeta$ is given by (where $\rho^\beta(\bs, \ba) := d^\behavior_{\mdphat}(\bs, \ba)$):
\begin{align*}
&\mathcal{O}\left(\frac{\gamma f}{(1 - \gamma)^2}\right) {\left[ \E_{\bs \sim d^{\pi_{\text{out}}}_{\mdp}}\left[ \sqrt{\frac{|\actions|}{|\data(\bs)|} (\mathrm{D}_{\text{CQL}}({\pi}_{\text{out}}, \behavior) + 1)} \right] \right]}\\
&+ \mathcal{O}\left(\frac{\gamma (1 - f)}{(1 - \gamma)^2}\right) {\mathrm{D_{TV}}(P_{\mdp}, P_{\mdphat})} - \beta \frac{C}{(1 - \gamma)}.
\end{align*}
\end{theorem}

\begin{proof}
We first note that since policy improvement is not being performed in the same MDP, $\mdp$ as the f-interpolant MDP, $\mdp_f$, we need to upper and lower bound the amount of improvement occurring in the actual MDP due to the f-interpolant MDP. As a result our first is to relate $J(\pi, \mdp)$ and $J(\pi, \mdp_f) := J(\mdpbar, \mdphat, f, \pi)$ for any given policy $\pi$.

\textbf{Step 1: Bounding the return in the actual MDP due to optimization in the f-interpolant MDP.} By directly applying Lemma~\ref{lemma:interpolant_regular_bound} stated and proved previously, we obtain the following upper and lower-bounds on the return of a policy $\pi$:
\begin{equation*}
    J(\mdpbar, \mdphat, f, \pi) \in \left[ J(\pi, \mdp) - \alpha,~~ J(\pi, \mdp) + \alpha \right],
\end{equation*}
where $\alpha$ is shown in Equation~\ref{eqn:alpha_expr}. As a result, we just need to bound the terms appearing the expression of $\alpha$ to obtain a bound on the return differences. We first note that the terms in the expression for $\alpha$ are of two types: \textbf{(1)} terms that depend only on the reward function differences (captured in $\Delta_R$ in Equation~\ref{eqn:alpha_expr_repeat}), and \textbf{(2)} terms that depend on the dynamics (the other two terms in Equation~\ref{eqn:alpha_expr_repeat}). 

To bound $\Delta_R$, we simply appeal to concentration inequalities on reward (Assumption~\ref{assumption:conc}), and bound $\Delta_R$ as:
\begin{align*}
\Delta_R &:= \frac{f}{1 - \gamma} \E_{\bs, \ba \sim d^\pi_{\mdp} \pi}[|r_{\mdp_1}(\bs, \ba) - r_{\mdp}(\bs, \ba)|] + \frac{1 - f}{1 - \gamma} \E_{\bs, \ba \sim d^\pi_{\mdp} \pi}[|r_{\mdp_2}(\bs, \ba) - r_{\mdp}(\bs, \ba)|]\\
&\leq \frac{C_{r, \delta}}{1 - \gamma} \E_{\bs, \ba \sim d^\pi_{\mdp}\pi} \left[\frac{1}{\sqrt{D(\bs, \ba)}}\right] + \frac{1}{1 - \gamma} ||R_{\mdp} - R_{\mdphat}|| := \Delta_R^u.
\end{align*}
Note that both of these terms are of the order of $\mathcal{O}(1/ (1 - \gamma))$ and hence they don't figure in the informal bound in Theorem~\ref{thm:policy_improvement} in the main text, as these are dominated by terms that grow quadratically with the horizon.
To bound the remaining terms in the expression for $\alpha$, we utilize a result directly from \citet{kumar2020conservative} for the empirical MDP, $\mdpbar$, which holds for any policy $\pi(\ba|\bs)$, as shown below.
\begin{align*}
   &\frac{\gamma}{(1 - \gamma)} \left\vert \E_{\bs, \ba \sim d^\pi_{\mdp}(\bs) \pi(\ba|\bs)}\left[ \left(P^\pi_{\mdp} - P^\pi_{\mdp_1}\right) Q^\pi_{\mdp} \right] \right\vert \\
   &\leq \frac{2 \gamma R_{\max} C_{P, \delta}}{(1 - \gamma)^2} \mathbb{E}_{\bs \sim d^{\policy}_{\mdpbar}(\bs)}\left[ \frac{\sqrt{|\mathcal{A}|}}{\sqrt{|\mathcal{D}(\bs)|}} \sqrt{ D_{\text{CQL}}(\policy, \behavior)(\bs) + 1} \right].
\end{align*}

\textbf{Step 2: Incorporate policy improvement in the f-inrerpolant MDP.} Now we incorporate the improvement of policy $\pi_{\text{out}}$ over the policy $\behavior$ on a weighted mixture of $\mdphat$ and $\mdpbar$. In what follows, we derive a lower-bound on this improvement by using the fact that policy $\pi_{\text{out}}$ is obtained by maximizing $\hat{J}(f, \pi)$ from Equation~\ref{eqn:penalized_objective}. As a direct consequence of Equation~\ref{eqn:penalized_objective}, we note that 
\begin{equation}
\label{eqn:improvement_expanded}
    \hat{J}(f, \pi_{\text{out}}) =  J(\mdpbar, \mdphat, f, \pi_{\text{out}}) - \beta \frac{\nu(\rho^\pi, f)}{1 - \gamma} \geq \hat{J}(f, \behavior) =  J(\mdpbar, \mdphat, f, \behavior) - \beta {\frac{\nu(\rho^\beta, f)}{1 - \gamma}}
\end{equation}

Following \textbf{Step 1}, we will use the upper bound on $J(\mdpbar, \mdphat, f, \pi)$ for policy $\pi = \pi_{\text{out}}$ and a lower-bound on $J(\mdpbar, \mdphat, f, \pi)$ for policy $\pi = \behavior$ and obtain the following inequality:
\begin{align*}
    J(\pi_{\text{out}}, \mdp) - \beta \frac{\nu(\rho^\pi, f)}{1 - \gamma} ~&\geq~ \Big\{ J(\behavior, \mdp) - \beta \frac{\nu(\rho^\beta, f)}{1 - \gamma}
    - \frac{4 \gamma (1 - f) R_{\max}}{(1 - \gamma)^2} D(P_{\mdp}, P_{\mdphat}) \\ 
    &- \underbrace{\frac{2 \gamma f}{(1 - \gamma)}\left\vert\E_{d^{\pi_{\text{out}}}_{\mdp}} \left[ \left(P^{\pi_{\text{out}}}_{\mdp} - P^{\pi_{\text{out}}}_{\mdpbar}\right) Q^{\pi_{\text{out}}}_{\mdp}  \right] \right\vert}_{:= (*)}\nonumber\\
    &- \underbrace{\frac{4 \gamma R_{\max} C_{P, \delta} f}{(1 - \gamma)^2} \E_{\bs \sim d^\behavior_{\mdp}}\left[ \sqrt{\frac{|\actions|}{|\data(\bs)|}}\right]}_{:= (\wedge)} - \Delta_R^u \Big\}.
\end{align*}
The term marked by $(*)$ in the above expression can be upper bounded by the concentration properties of the dynamics as done in Step 1 in this proof: 
\begin{align}
\label{eqn:bound_mdp_mdphat}
    (*) \leq \frac{4 \gamma f C_{P, \delta} R_{\max}}{(1 - \gamma)^2} \mathbb{E}_{\bs \sim d^{{\pi_{\text{out}}}}_{\mdp}(\bs)}\left[ \frac{\sqrt{|\mathcal{A}|}}{\sqrt{|\mathcal{D}(\bs)|}} \sqrt{ D_{\text{CQL}}({\pi_{\text{out}}}, \behavior)(\bs) + 1} \right]. 
\end{align}
Finally, using Equation~\ref{eqn:bound_mdp_mdphat}, we can lower-bound the policy return difference as:
\begin{align*}
    J(\pi_{\text{out}}, \mdp) - J(\behavior, \mdp) &\geq \beta \frac{\nu(\rho^\pi, f)}{1 - \gamma} - \beta \frac{\nu(\rho^\beta, f)}{1 - \gamma} - \frac{4 \gamma (1 -f) R_{\max}}{(1 - \gamma)^2} D(P_{\mdp}, P_{\mdphat}) - (*) - \Delta_R^u\\
    &\geq \beta \frac{C}{1 - \gamma} - \frac{4 \gamma (1 -f) R_{\max}}{(1 - \gamma)^2} D(P_{\mdp}, P_{\mdphat}) - (*) - \Delta_R^u.
\end{align*}
Plugging the bounds for terms (a), (b) and (c) in the expression for $\zeta$ where $J(\pi_{\text{out}}, \mdp) - J(\behavior, \mdp) \geq \zeta$, we obtain:
\begin{align}
\zeta &= \left({\frac{4f \gamma R_{\max} C_{P, \delta}}{(1 - \gamma)^2}} \right)\mathbb{E}_{\bs \sim d^{\policy_{\text{out}}}_{\mdp}(\bs)}\left[ \frac{\sqrt{|\mathcal{A}|}}{\sqrt{|\mathcal{D}(\bs)|}} \sqrt{ D_{\text{CQL}}(\policy_{\text{out}}, \behavior)(\bs) + 1} \right]  + (\wedge) - \Delta_R^u \nonumber\\
\label{eqn:zeta_expression}
&~~~~~~~~~~~~+ \frac{4 (1 -f) \gamma R_{\max}}{(1 - \gamma)^2} D(P_{\mdp}, P_{\mdphat}) - \beta \frac{C}{1 - \gamma}.
\end{align}
\end{proof}

\begin{remark}[\underline{\textbf{Interpretation of Proposition~\ref{thm:policy_improvement}}}] 
\label{remark:remark1}
Now we will interpret the theoretical expression for $\zeta$ in Equation~\ref{eqn:zeta_expression}, and discuss the scenarios when it is \emph{negative}. When the expression for $\zeta$ is negative, the policy $\pi_{\text{out}}$ is an improvement over $\behavior$ in the original MDP, $\mdp$. 

\begin{itemize}
    \item We first discuss if the assumption of $\nu(\rho^{\pi_{\text{out}}}, f) - \nu(\rho^\beta, f) \geq C > 0$ is reasonable in practice. Note that we have never used the fact that the learned model $P_{\mdphat}$ is close to the actual MDP, $P_{\mdp}$ on the states visited by the behavior policy $\behavior$ in our analysis. We will use this fact now: in practical scenarios, $\nu(\rho^\beta, f)$ is expected to be smaller than $\nu(\rho^\pi, f)$, since $\nu(\rho^\beta, f)$ is directly controlled by the difference and density ratio of $\rho^\beta(\bs, \ba)$ and $d(\bs, \ba)$: $\nu(\rho^\beta, f) \leq \nu(\rho^\beta, f=1) = \sum_{\bs, \ba} d^\behavior_{\mdphat}(\bs, \ba) \left(d^\behavior_{\mdphat}(\bs, \ba)/d^\behavior_{\mdpbar}(\bs, \ba) - 1\right)^2$ by Lemma~\ref{thm:line_thm} which is expected to be small for the behavior policy $\behavior$ in cases when the behavior policy marginal in the empirical MDP, $d^\behavior_{\mdpbar}(\bs, \ba)$, is broad. This is a direct consequence of the fact that the learned dynamics integrated with the policy under the learned model: $P_{\mdphat}^\behavior$ is closer to its counterpart in the empirical MDP:  $P_{\mdpbar}^\behavior$ for $\behavior$. Note that this is not true for any other policy besides the behavior policy that performs several counterfactual actions in a rollout and deviates from the data. For such a learned policy $\pi$, we incur an extra error which depends on the importance ratio of policy densities, compounded over the horizon and manifests as the $D_{\mathrm{CQL}}$ term (similar to Equation~\ref{eqn:bound_mdp_mdphat}, or Lemma D.4.1 in \citet{kumar2020conservative}). Thus, in practice, we argue that we are interested in situations where the assumption $\nu(\rho^{\pi_{\text{out}}}, f) - \nu(\rho^\beta, f) \geq C > 0$ holds, in which case by increasing $\beta$, we can make the expression for $\zeta$ in Equation~\ref{eqn:zeta_expression} negative, allowing for policy improvement.
    \item In addition, note that when $f$ is close to 1, the bound reverts to a standard model-free policy improvement bound and when $f$ is close to 0, the bound reverts to a typical model-based policy improvement bound. In scenarios with high sampling error (i.e. smaller $|\mathcal{D}(\bs)|$), if we can learn a good model, i.e., $D(P_{\mdp}, P_{\mdphat})$ is small, we can attain policy improvement better than model-free methods by relying on the learned model by setting $f$ closer to 0. A similar argument can be made in reverse for handling cases when learning an accurate dynamics model is hard. 
\end{itemize}
\end{remark}

\section{Experimental details}
\label{app:details}

In this section, we include all details of our empirical evaluations of COMBO.

\subsection{Practical algorithm implementation details}
\label{app:combo_details}

\paragraph{Model training.}

In the setting where the observation space is low-dimensional, as mentioned in Section~\ref{sec:combo},  we represent the model as a probabilistic neural network that outputs a Gaussian distribution over the next state and reward given the current state and action: $$\widehat{T}_\theta(\bs_{t+1}, r| \bs, \ba) = \mathcal{N}(\mu_\theta(\bs_t, \ba_t), \Sigma_\theta(\bs_t, \ba_t)).$$ We train an ensemble of $7$ such dynamics models following \cite{janner2019trust} and pick the best $5$ models based on the validation prediction error on a held-out set that contains $1000$ transitions in the offline dataset $\data$. During model rollouts, we randomly pick one dynamics model from the best $5$ models. Each model in the ensemble is represented as a 4-layer feedforward neural network with $200$ hidden units. For the generalization experiments in Section~\ref{sec:generalization_exps}, we additionally use a two-head architecture to output the mean and variance after the last hidden layer following \cite{yu2020mopo}.

In the image-based setting, we follow \citet{Rafailov2020LOMPO} and use a variational model with the following components:

\begin{gather}
\begin{aligned}
&\text{Image encoder:} && \mathbf{h}_t=E_\theta(\bo_t) \\
&\text{Inference model:} && \bs_t \sim q_\theta(\bs_t|\mathbf{h}_t, \bs_{t-1}, \ba_{t-1})\\
&\text{Latent transition model:} &&\bs_t \sim \widehat{T}_\theta(\bs_t| \bs_{t-1}, \ba_{t-1})\\
&\text{Reward predictor:} && r_t \sim p_\theta(r_t|\bs_t) \\
&\text{Image decoder:} && \bo_t \sim D_\theta(\bo_t|\bs_t).
\label{eq:latent_model}
\end{aligned}
\end{gather}%

We train the model using the evidence lower bound:

$$\max_{\theta}\sum_{\tau=0}^{T-1}\Big[\mathbb{E}_{q_{\theta}}[\log D_{\theta}(\bo_{\tau+1}|\bs_{\tau+1})]\Big]-\mathbb{E}_{q_{\theta}}\Big[D_{KL}[q_{\theta}(\bo_{\tau+1}, \bs_{\tau+1}|\bs_{\tau}, \ba_{\tau})\|\widehat{T}_{\theta_{\tau}}(\bs_{\tau+1}, a_{\tau+1})]\Big]$$

At each step $\tau$ we sample a latent forward model $\widehat{T}_{\theta_{\tau}}$ from a fixed set of $K$ models $[\widehat{T}_{\theta_1},\ldots, \widehat{T}_{\theta_K}]$. For the encoder $E_{\theta}$ we use a convolutional neural network with kernel size 4 and stride 2. For the Walker environment we use 4 layers, while the Door Opening task has 5 layers. The $D_{\theta}$ is a transposed convolutional network with stride 2 and kernel sizes $[5,5,6,6]$ and $[5,5,5,6,6]$ respectively. The inference network has a two-level structure similar to \citet{Hafner2019PlanNet} with a deterministic path using a GRU cell with 256 units and a stochastic path implemented as a conditional diagonal Gaussian with 128 units. We only train an ensemble of stochastic forward models, which are also implemented as conditional diagonal Gaussians.

\paragraph{Policy Optimization.} We sample a batch size of $256$ transitions for the critic and policy learning. We set $f = 0.5$, which means we sample $50\%$ of the batch of transitions from $\data$ and another $50\%$ from $\data_\text{model}$. The equal split between the offline data and the model rollouts strikes the balance between conservatism and generalization in our experiments as shown in our experimental results in Section~\ref{sec:exp}. We represent the Q-networks and policy as 3-layer feedforward neural networks with $256$ hidden units.

For the choice of $\rho(\bs,\ba)$ in Equation~\ref{eq:implicit_update}, we can obtain the Q-values that lower-bound the true value of the learned policy $\pi$ by setting $\rho(\bs,\ba) = d^\policy_{\mdphat} (\bs) \pi(\ba | \bs)$. However, as discussed in \cite{kumar2020conservative}, computing $\pi$ by alternating the full off-policy evaluation for the policy $\hat{\pi}^k$ at each iteration $k$ and one step of policy improvement is computationally expensive. Instead, following \cite{kumar2020conservative}, we pick a particular distribution $\psi(\ba|\bs)$ that approximates the the policy that maximizes the Q-function at the current iteration and set $\rho(\bs,\ba) = d^\policy_{\mdphat} (\bs) \psi(\ba | \bs)$. We formulate the new objective as follows:
\begin{small}
\begin{align}
    \hat{Q}^{k+1} \leftarrow& \arg\min_{Q}\beta\left(\E_{\bs \sim d^\policy_{\mdphat} (\bs), \ba\sim \psi(\ba | \bs)}\!\left[Q(\bs,\ba)\right]-\E_{\bs, \ba \sim \data}\left[Q(\bs,\ba)\right]\right)\nonumber\\
    &+ \frac{1}{2}\E_{\bs, \ba, \bs' \sim d_f}\left[ \left(Q(\bs, \ba) - \widehat{\bellman}^\policy\hat{Q}^k(\bs, \ba))\right)^2 \right] + \mathcal{R}(\psi),
    \label{eq:combo_update_practical}
\end{align}
\end{small}
where $\mathcal{R}(\psi)$ is a regularizer on $\psi$. In practice, we pick $\mathcal{R}(\psi)$ to be the $-D_\text{KL}(\psi(\ba|\bs)\|\text{Unif}(\ba))$ and under such a regularization, the first term in Equation~\ref{eq:combo_update_practical} corresponds to computing softmax of the Q-values at any state $\bs$ as follows:
\begin{small}
\begin{align}
    \hat{Q}^{k+1} \leftarrow& \arg\min_{Q}\max_\psi\beta\left(\E_{\bs \sim d^\policy_{\mdphat} (\bs)}\!\left[\log\sum_\ba Q(\bs,\ba)\right]-\E_{\bs, \ba \sim \data}\left[Q(\bs,\ba)\right]\right) \nonumber\\
    &+ \frac{1}{2}\E_{\bs, \ba, \bs' \sim d_f}\left[ \left(Q(\bs, \ba) - \widehat{\bellman}^\policy\hat{Q}^k(\bs, \ba))\right)^2 \right].
    \label{eq:combo_logsumexp}
\end{align}
\end{small}
We estimate the \texttt{log-sum-exp} term in Equation~\ref{eq:combo_logsumexp} by sampling $10$ actions at every state $\bs$ in the batch from a uniform policy $\text{Unif}(\ba)$ and the current learned policy $\pi(\ba|\bs)$ with importance sampling following \cite{kumar2020conservative}.

\subsection{Hyperparameter Selection}
\label{app:hyperparameter}

\neurips{In this section, we discuss the hyperparameters that we use for COMBO. In the D4RL and generalization experiments, our method are built upon the implementation of MOPO provided at: \url{https://github.com/tianheyu927/mopo}. The hyperparameters used in COMBO that relates to the backbone RL algorithm SAC such as twin Q-functions and number of gradient steps follow from those used in MOPO with the exception of smaller critic and policy learning rates, which we will discuss below. In the image-based domains, COMBO is built upon LOMPO without any changes to the parameters used there. For the evaluation of COMBO, we follow the evaluation protocol in D4RL~\citep{fu2020d4rl} and a variety of prior offline RL works~\citep{kumar2020conservative,yu2020mopo,kidambi2020morel} and report the normalized score of the smooth undiscounted averaged return over $3$ random seeds for all environments except \texttt{sawyer-door-close} and \texttt{sawyer-door} where we report the average success rate over $3$ random seeds.} \iclr{As mentioned in Section~\ref{sec:combo}, we use the regularization objective in Eq.~\ref{eq:implicit_update} to select the hyperparameter from a range of pre-specified candidates in a fully offline manner, unlike prior model-based offline RL schemes such as \cite{yu2020mopo} and \cite{kidambi2020morel} that similar hyperparameters as COMBO and tune them manually based on policy performance obtained via online rollouts.}

\neurips{We now list the additional hyperparameters as follows.
\begin{itemize}
    \item \textbf{Rollout length $h$.} We perform a short-horizon model rollouts in COMBO similar to \citet{yu2020mopo} and \citet{Rafailov2020LOMPO}. For the D4RL experiments and generalization experiments, we followed the defaults used in MOPO and used $h = 1$ for walker2d and \texttt{sawyer-door-close}, $h=5$ for hopper, halfcheetah and \texttt{halfcheetah-jump}, and $h=25$ for \texttt{ant-angle}. In the image-based domain we used rollout length of $h=5$ for both the the \texttt{walker-walk} and \texttt{sawyer-door-open} environments following the same hyperparameters used in \citet{Rafailov2020LOMPO}.
    \item \textbf{Q-function and policy learning rates.} On state-based domains, we apply our automatic selection rule to the set $\{1e-4, 3e-4\}$ for the Q-function learning rate and the set $\{1e-5, 3e-5, 1e-4\}$ for the policy learning rate. 
    We found that $3e-4$ for the Q-function learning rate (also used previously in \citet{kumar2020conservative}) and $1e-4$ for the policy learning rate (also recommended previously in \citet{kumar2020conservative} for gym domains) work well for almost all domains except that on walker2d where a smaller Q-function learning rate of $1e-4$ and a correspondingly smaller policy learning rate of $1e-5$ works the best according to our automatic hyperparameter selection scheme. In the image-based domains, we followed the defaults from prior work \citep{Rafailov2020LOMPO} and used $3e-4$ for both the policy and Q-function.
    
    \item \textbf{Conservative coefficient $\beta$.} 
    We \iclr{use our hyperparameter selection rule to select the right $\beta$ from the set} $\{0.5, 1.0, 5.0\}$ for $\beta$, which correspond to low conservatism, medium conservatism and high conservatism.  A larger $\beta$ would be desirable in more narrow dataset distributions with lower-coverage of the state-action space that propagates error in a backup whereas a smaller $\beta$ is desirable with diverse dataset distributions. On the D4RL experiments, we found that $\beta = 0.5$ works well for halfcheetah agnostic of dataset quality, while on hopper and walker2d, we found that the more ``narrow'' dataset distributions: medium and medium-expert datasets work best with larger $\beta = 5.0$ whereas more ``diverse'' dataset distributions: random and medium-replay datasets work best with smaller $\beta=0.5$ which is consistent with the intuition. 
    On generalization experiments, $\beta = 1.0$ works best for all environments. In the image-domains we use $\beta=0.5$ for the medium-replay \texttt{walker-walk} task and and $\beta=1.0$ for all other domains, which again is in accordance with the impact of $\beta$ on performance.

    \item \textbf{Choice of $\rho(\bs,\ba)$.} We first decouple $\rho(\bs,\ba) = \rho(\bs)\rho(\ba|\bs)$ for convenience. As discussed in Appendix~\ref{app:combo_details}, we use $\rho(\ba|\bs)$ as the soft-maximum of the Q-values and estimated with \texttt{log-sum-exp}. For $\rho(\bs)$, \iclr{we apply the automatic hyperparameter selection rule to the set} $\{d^\policy_{\mdphat}, \rho(\bs)=d_f\}$.  We found that $d^\policy_{\mdphat}$ works better the hopper task in D4RL while $d_f$ is better for the rest of the environments. For the remaining domains, we found $\rho(\bs)=d_f$ works well.

    \item \textbf{Choice of $\mu(\ba|\bs)$.} For the rollout policy $\mu$, we \iclr{use our automatic selection rule on the set} $\{\text{Unif}(\ba), \pi(\ba|\bs)\}$, i.e. the set that contains a random policy and a current learned policy. We found that $\mu(\ba|\bs) = \text{Unif}(\ba)$ works well on the hopper task in D4RL and also in the $\texttt{ant-angle}$ generalization experiment. For the remaining state-based environments, we discovered that $\mu(\ba|\bs) = \pi(\ba|\bs)$ excels. In the image-based domain, we found that $\mu(\ba|\bs) = \text{Unif}(\ba)$ works well in the \texttt{walker-walk} domain and  $\mu(\ba|\bs) = \pi(\ba|\bs)$ is better for the \texttt{sawyer-door} environment. 
    We observed that
    $\mu(\ba|\bs) = \text{Unif}(\ba)$ behaves less conservatively and is suitable to tasks where dynamics models can be learned fairly precisely.
    \item \textbf{Choice of $f$.} For the ratio between model rollouts and offline data $f$, we \iclr{input the set} $\{0.5, 0.8\}$ \iclr{to our automatic hyperparameter selection rule to figure out the best $f$ on each domain}. We found that $f = 0.8$ works well on the medium and medium-expert in the walker2d task in D4RL. For the remaining environments, we find $f = 0.5$ works well.
\end{itemize}}

\iclr{We also provide additional experimental results on how our automatic hyperparameter selection rule selects hyperparameters. As shown in Table~\ref{tab:beta_selection},~\ref{tab:mu_selection}, ~\ref{tab:rho_selection} and \ref{tab:f_selection}, our automatic hyperparameter selection rule is able to pick the hyperparameters $\beta$, $\mu(\ba|\bs)$, $\rho(\bs)$ and $f$ and  that correspond to the best policy performance based on the regularization value.}

\begin{table}[ht]
    \centering
    \scriptsize
    \begin{tabular}{l|r|r|r|r|}
    \toprule
    Task & $\beta=0.5$ & $\beta=0.5$ & $\beta=5.0$ & $\beta=5.0$\\
 & performance & regularizer value & performance & regularizer value\\
 \midrule
halfcheetah-medium &  \textbf{54.2}  & \textbf{-778.6}  & 40.8  & -236.8  \\
halfcheetah-medium-replay &  \textbf{55.1} & \textbf{28.9} & 9.3 & 283.9\\ 
halfcheetah-medium-expert & 89.4 & 189.8 & \textbf{90.0}  & \textbf{6.5}\\
hopper-medium      &  75.0  & -740.7  &\textbf{97.2}  & \textbf{-2035.9}\\
hopper-medium-replay & \textbf{89.5} & \textbf{37.7} & 28.3       & 107.2\\
hopper-medium-expert & \textbf{111.1}       & \textbf{-705.6}    & 75.3 &       -64.1\\
walker2d-medium        &  1.9  & 51.5  & \textbf{81.9}  & \textbf{-1991.2}\\
walker2d-medium-replay & \textbf{56.0}       & \textbf{-157.9}    & 27.0       & 53.6\\
walker2d-medium-expert & 10.3       & -788.3    &\textbf{103.3}       & \textbf{-3891.4}\\
    \bottomrule
    \end{tabular}
    \caption{\footnotesize We include our automatic hyperparameter selection rule of $\beta$ on a set of representative D4RL environments. We show the policy performance (bold with the higher number) and the regularizer value (bold with the lower number). Lower regularizer value consistently corresponds to the higher policy return, suggesting the effectiveness of our automatic selection rule.}
    \label{tab:beta_selection}
\end{table}

\begin{table}[ht]
    \centering
    \scriptsize
    \begin{tabular}{l|r|r|r|r|}
    \toprule
    Task & $\mu(\ba|\bs)=\text{Unif}(\ba)$ & $\mu(\ba|\bs)=\text{Unif}(\ba)$            &$\mu(\ba|\bs)=\pi(\ba|\bs)$&$\mu(\ba|\bs)=\pi(\ba|\bs)$\\
 & performance & regularizer value & performance & regularizer value\\
 \midrule
hopper-medium        & \textbf{97.2}  & \textbf{-2035.9} &  52.6  & -14.9  \\
walker2d-medium        &  7.9  & -106.8  & \textbf{81.9}  & \textbf{-1991.2} \\
    \bottomrule
    \end{tabular}
    \caption{\footnotesize We include our automatic hyperparameter selection rule of $\mu(\ba|\bs)$ on the medium datasets in the hopper and walker2d environments from D4RL. We follow the same convention defined in Table~\ref{tab:beta_selection} and find that our automatic selection rule can effectively select $\mu$ offline.}
    \label{tab:mu_selection}
\end{table}

\begin{table}[ht]
    \centering
    \scriptsize
    \begin{tabular}{l|r|r|r|r|}
    \toprule
    Task & $\rho(\bs) = d^\pi_{\hat{\mathcal{M}}} $&$\rho(\bs) = d^\pi_{\hat{\mathcal{M}}}$            &$\rho(\bs) = d_f$&$\rho(\bs) = d_f$\\
 & performance & regularizer value & performance & regularizer value\\
 \midrule
hopper-medium        & \textbf{97.2}  & \textbf{-2035.9} &  56.0  & -6.0  \\
walker2d-medium        &  1.8  & 14617.4  & \textbf{81.9}  & \textbf{-1991.2} \\
    \bottomrule
    \end{tabular}
    \caption{\footnotesize We include our automatic hyperparameter selection rule of $\rho(\bs)$ on the medium datasets in the hopper and walker2d environments from D4RL. We follow the same convention defined in Table~\ref{tab:beta_selection} and find that our automatic selection rule can effectively select $\rho$ offline.}
    \label{tab:rho_selection}
\end{table}

\begin{table}[ht]
    \centering
    \scriptsize
    \begin{tabular}{l|r|r|r|r|}
    \toprule
    Task & $f = 0.5 $&$f = 0.5$            &$f = 0.8$&$f = 0.8$\\
 & performance & regularizer value & performance & regularizer value\\
 \midrule
hopper-medium        & \textbf{97.2}  & \textbf{-2035.9} &  93.8  & -21.3  \\
walker2d-medium        &  70.9  & -1707.0  & \textbf{81.9}  & \textbf{-1991.2} \\
    \bottomrule
    \end{tabular}
    \caption{\footnotesize We include our automatic hyperparameter selection rule of $f$ on the medium datasets in the hopper and walker2d environments from D4RL. We follow the same convention defined in Table~\ref{tab:beta_selection} and find that our automatic selection rule can effectively select $f$ offline.}
    \label{tab:f_selection}
\end{table}

\subsection{Details of generalization environments}
\label{app:ood_details}

For \texttt{halfcheetah-jump} and \texttt{ant-angle}, we follow the same environment used in MOPO. For \texttt{sawyer-door-close}, we train the \texttt{sawyer-door} environment in \url{https://github.com/rlworkgroup/metaworld} with dense rewards for opening the door until convergence. We collect $50000$ transitions with half of the data collected by the final expert policy and a policy that reaches the performance of about half the expert level performance. We relabel the reward such that the reward is $1$ when the door is fully closed and $0$ otherwise. Hence, the offline RL agent is required to learn the behavior that is different from the behavior policy in a sparse reward setting. We provide the datasets in the following anonymous link\footnote{The datasets of the generalization environments are available at the anonymous link: \url{https://drive.google.com/file/d/1pn6dS5OgPQVp_ivGws-tmWdZoU7m_LvC/view?usp=sharing}.}.

\subsection{Details of image-based environments}
\label{app:image_details}

\begin{figure}[ht]
    \centering
    \includegraphics[width=0.25\textwidth]{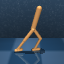}
    \includegraphics[width=0.25\textwidth]{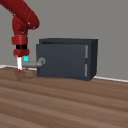}
    \vspace{-0.2cm}
    \caption{\footnotesize Our image-based environments: The observations are $64\times 64$ and $128\times 128$ raw RGB images for the \texttt{walker-walk} and \texttt{sawyer-door} tasks respectively. The \texttt{sawyer-door-close} environment used in in Section~\ref{sec:generalization_exps} also uses the \texttt{sawyer-door} environment.}
    \label{fig:visual}
\end{figure}

We visualize our image-based environments in Figure~\ref{fig:visual}. We use the standard \texttt{walker-walk} environment from \citet{tassa2018deepmind} with $64\times64$ pixel observations and an action repeat of 2. Datasets were constructed the same way as \citet{fu2020d4rl} with 200 trajectories each. For the \texttt{sawyer-door} we use $128\times128$ pixel observations. The medium-expert dataset contains 1000 rollouts (with a rollout length of 50 steps) covering the state distribution from grasping the door handle to opening the door. The expert dataset contains 1000 trajectories samples from a fully trained (stochastic) policy. The data was obtained from the training process of a stochastic SAC policy using dense reward function as defined in \citet{yu2020meta}. However, we relabel the rewards, so an agent receives a reward of 1 when the door is fully open and 0 otherwise. This aims to evaluate offline-RL performance in a sparse-reward setting. All the datasets are from \citep{Rafailov2020LOMPO}.

\subsection{Computation Complexity}

For the D4RL and generalization experiments, COMBO is trained on a single NVIDIA GeForce RTX 2080 Ti for one day. For the image-based experiments, we utilized a single NVIDIA GeForce RTX 2070. We trained the \texttt{walker-walk} tasks for a day and the \texttt{sawyer-door-open} tasks for about two days.

\subsection{License of datasets}

We acknowledge that all datasets used in this paper use the MIT license.

\begin{table}[ht]
\centering
\scriptsize
\begin{tabular}{l|r|r|r|r|}
\toprule 
\textbf{Environment} & \stackanchor{\textbf{Batch}}{\textbf{Mean}} & \stackanchor{\textbf{Batch}}{\textbf{Max}} & \stackanchor{\textbf{COMBO}}{\textbf{(Ours)}} & \textbf{CQL+MBPO}\\ \midrule
halfcheetah-jump & -1022.6 & 1808.6 & \textbf{5392.7}$\pm$575.5 & 4053.4$\pm$176.9\\
ant-angle & 866.7 & 2311.9 & \textbf{2764.8}$\pm$43.6 & 809.2$\pm$135.4\\
sawyer-door-close & 5\% & 100\% & \textbf{100}\%$\pm$0.0\% & 62.7\%$\pm$24.8\%\\
\bottomrule
\end{tabular}
\caption{
\footnotesize Comparison between COMBO and CQL+MBPO on tasks that require out-of-distribution generalization. Results are in average returns of \texttt{halfcheetah-jump} and \texttt{ant-angle} and average success rate of \texttt{sawyer-door-close}. All results are averaged over 6 random seeds, $\pm$ the $95\%$-confidence interval.
}
\vspace{-0.3cm}
\label{tbl:cql_mbpo}
\normalsize
\end{table}

\section{Comparison to the Naive Combination of CQL and MBPO}
\label{app:cql_mbpo}

\iclr{In this section, we stress the distinction between COMBO and a direct combination of two previous methods CQL and MBPO (denoted as CQL + MBPO). CQL+MBPO performs Q-value regularization using CQL while expanding the offline data with MBPO-style model rollouts. While COMBO utilizes Q-value regularization similar to CQL, the effect is very different. CQL only penalizes the Q-value on unseen actions on the states observed in the dataset whereas COMBO penalizes Q-values on states generated by the learned model while maximizing Q values on state-action tuples in the dataset. Additionally, COMBO also utilizes MBPO-style model rollouts for also augmenting samples for training Q-functions.

To empirically demonstrate the consequences of this distinction, CQL + MBPO performs quite a bit worse than COMBO on generalization experiments (Section~\ref{sec:generalization_exps}) as shown in Table~\ref{tbl:cql_mbpo}. The results are averaged across 6 random seeds ($\pm$ denotes 95\%-confidence interval of the various runs). This suggests that carefully considering the state distribution, as done in COMBO, is crucial.}

\end{document}